%% file: neurips_2026.tex
\documentclass{article}

\usepackage[preprint]{neurips_2026}

\usepackage[utf8]{inputenc} 
\usepackage[T1]{fontenc}    
\usepackage{hyperref}       
\usepackage{url}            
\usepackage{booktabs}       
\usepackage{amsfonts}       
\usepackage{nicefrac}       
\usepackage{microtype}      
\usepackage{xcolor}         

\usepackage{algorithmic}
\usepackage{amssymb}
\usepackage{mathtools}
\usepackage{amsthm}
\usepackage[dvipsnames]{xcolor}
\usepackage{colortbl}
\usepackage{soul}
\usepackage{longtable}
\usepackage{multirow}
\usepackage{multicol}
\usepackage{xspace}
\usepackage{pifont}
\usepackage{latexsym}
\usepackage{inconsolata}
\usepackage{arydshln}
\usepackage{enumitem}
\usepackage{wrapfig}
\usepackage{array}
\usepackage{wrapfig}
\usepackage{tcolorbox}
\usepackage{amsmath}
\usepackage{amssymb} 
\usepackage{amsfonts}

\usepackage{amsthm}

\newtheorem{theorem}{Theorem}

\newcommand{\comt}[1]{#1}

\renewcommand{\comt}[1]{}

\newcommand{\up}[1]{\textcolor{OliveGreen}{\footnotesize \ $\uparrow${#1}}}
\newcommand{\downbad}[1]{\textcolor{Maroon}{\footnotesize \ $\downarrow${#1}}}
\newcommand{\down}[1]{\textcolor{OliveGreen}{\footnotesize \ $\downarrow${#1}}}
\newcommand{\upbad}[1]{\textcolor{Maroon}{\footnotesize \ $\uparrow${#1}}}
\newcommand{\basex}[1]{\textcolor{gray!50}{\footnotesize \ $\uparrow${#1}}}
\newcommand{\basexx}[1]{\textcolor{gray!85}{\footnotesize \ $\uparrow${#1}}}
\newcommand{\basexdown}[1]{\textcolor{gray!50}{\footnotesize \ $\downarrow${#1}}}

\newcommand{\cmark}{\textcolor{darkgreen}{\ding{51}}} %
\newcommand{\xmark}{\textcolor{darkred}{\ding{55}}} %
\newcommand{\sexyname}{LaSCD\xspace}

\definecolor{mydarkblue}{rgb}{0,0.08,0.45}
\definecolor{darkgreen}{rgb}{0.0, 0.5, 0.0} 
\definecolor{myblue}{RGB}{235,235,250}
\definecolor{lightpink}{RGB}{222, 235, 225} 
\definecolor{lightblue}{RGB}{230, 235, 245} 
\definecolor{lightgray}{RGB}{240, 240, 240} 
\definecolor{darkgray}{RGB}{220, 220, 220} 

\definecolor{superlightred}{rgb}{0.99, 0.92, 0.92}
\definecolor{darkgreen}{RGB}{50,100,0}
\definecolor{darkred}{RGB}{200, 0, 0}

\title{When Looking Is Not Enough: Visual Attention Structure Reveals Hallucination in MLLMs}

\author{%
  Fanpu Cao$^{1}$, Xin Zou$^{1}$, Xuming Hu$^{1*}$, Hui Xiong$^{1,2}$\thanks{Corresponding author.}\\
  Thrust of Artificial Intelligence, HKUST (Guangzhou), China \\
  Department of Computer Science and Engineering, HKUST, Hong Kong SAR, China \\
  \texttt{fcao628@connect.hkust-gz.edu.cn; xuminghu@hkust-gz.edu.cn; xionghui@ust.hk} \\
}

\begin{document}

\maketitle

\begin{abstract}
Multimodal large language models (MLLMs) have become a key interface for visual reasoning and grounded question answering, yet they remain vulnerable to visual hallucinations, where generated responses contradict image content or mention nonexistent objects. A central challenge is that hallucination is not always caused by a simple lack of visual attention: the model may still assign substantial attention mass to image tokens while internally drifting toward an incorrect answer. In this paper, we show that the high-frequency structure of visual attention, measured by layer-wise Laplacian energy, reveals both the layer where hallucinated preferences emerge and the layer where the ground-truth answer transiently recovers. Building on this finding, we propose \textbf{\sexyname} (\textbf{La}placian-\textbf{S}pectral \textbf{C}ontrastive \textbf{D}ecoding), a training-free decoding strategy that selects informative layers via Laplacian energy and remaps next-token logits in closed form. Experiments on hallucination and general multimodal benchmarks show that \sexyname{} consistently reduces hallucination while preserving general capabilities, highlighting its potential as a faithful decoding paradigm. The code is available at \href{https://github.com/macovaseas/LaSCD}{\texttt{https://github.com/macovaseas/LaSCD}}.
\end{abstract}

\input{sections/introduction}

\input{sections/related_work}

\input{sections/finding}

\input{sections/method}

\input{sections/experiment}

\section{Limitations}
\label{sec:limitations}

\sexyname{} assumes that visual-token order preserves meaningful local structure. This is natural for patch-based encoders with grid-like token layouts, but may be less direct for architectures employing non-grid token aggregation strategies. In addition, our evaluation has not systematically covered reasoning-intensive tasks or substantially larger-scale MLLMs, leaving these settings for future study.

\input{sections/conclusion}

\bibliographystyle{unsrt}
\bibliography{references}

\newpage
\input{sections/appendix}



\end{document}

%% file: sections/introduction.tex
\section{Introduction}
\label{sec:introduction}


Multimodal large language models (MLLMs) have become a central interface for visual reasoning, grounded dialogue, and open-ended visual question answering~\cite{alayrac2022flamingo,liu2024visual,bai2023qwen,liu2024llavanext}. Despite this progress, they still suffer from visual hallucination: generated descriptions or answers can contradict the image, mention nonexistent objects, or make unsupported visual judgments~\cite{huang2024visualhall,zheng2024reefknot,lyu2025realrag}. This failure is especially concerning in safety sensitive applications such as healthcare and autonomous driving~\cite{lin2024hassurvey,ding2024holistic}. Existing mitigation methods address hallucination from different angles, including retrieval augmentation~\cite{qu2024alleviating,yang2024rag}, fine-tuning or preference optimization~\cite{yu2024rlhf,liu2024mitigating}, attention intervention~\cite{opera2024cvpr}, and contrastive decoding (CD)~\cite{vcd2024cvpr,li2023contrastive,shi2024trusting}. While these methods improve faithfulness in many cases, they often treat hallucination as a decoding artifact to be corrected, yet a more fundamental question remains open: \emph{what internal visual signal indicates that the model is about to commit to a hallucinated answer?}

A common intuition is that hallucination occurs because the decoder does not attend enough to the image. This intuition has motivated methods that increase, reallocate, or intervene on visual attention during decoding~\cite{opera2024cvpr,liu2024paying,zoulook}. However, visual attention magnitude alone may not answer the key diagnostic question (Sec. \ref{sec:finding1}). If the model can assign substantial attention to visual tokens while drifting toward an incorrect answer, then the problem is not simply whether the model ``looks'' at the image, but whether the structure of that visual attention supports faithful grounding. This gap has direct consequences for both analysis and mitigation: magnitude-based signals can activate at the wrong layer where the model is already visually engaged but internally misaligned.

\begin{figure}
    \centering
    \includegraphics[width=1.0\linewidth]{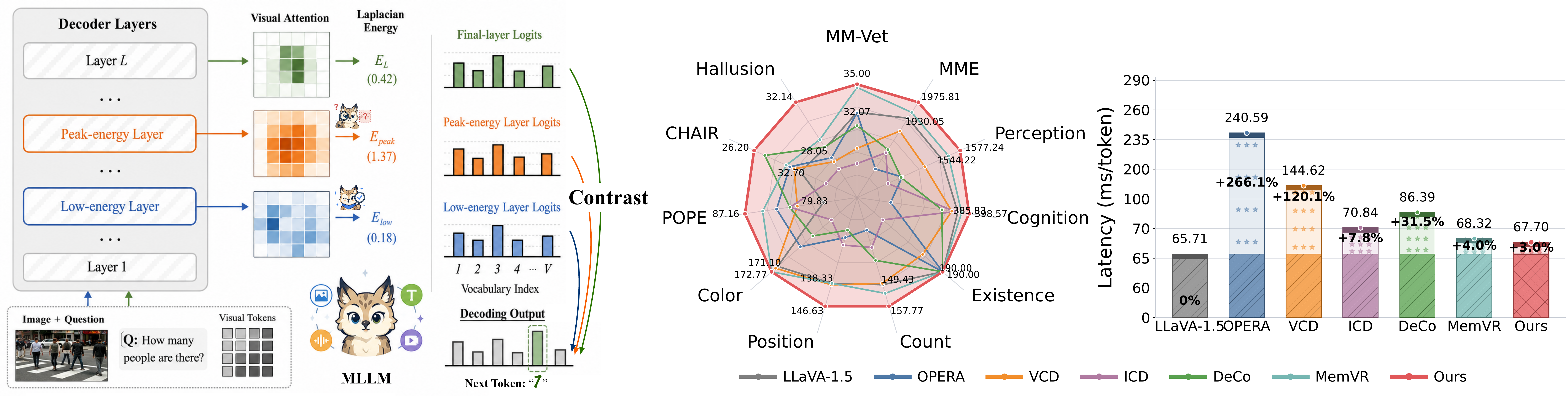}
    \caption{Overview of LaSCD \& Performance and latency comparison of SOTA. At each decoding step, LaSCD computes the discrete Laplacian energy of the last-query visual self-attention across candidate layers. It then selects the maximum- and minimum-energy layers to apply a lightweight contrastive logit remap, down-weighting the peak layer's predictions relative to the final head.}
    \label{fig:radar}
\end{figure}

In this paper, we present an empirical analysis showing that visual-attention \emph{magnitude} is a suboptimal proxy for hallucination, whereas the \emph{high-frequency structure} of visual attention provides a sharper layer-wise signal. Specifically, we compute a Laplacian energy on the last-query attention slice over visual tokens. Our analysis reveals two complementary phenomena (Fig. \ref{fig:case}, \ref{fig:case2}). First, hallucinated outputs can emerge when the visual-attention mass is still high, directly contradicting the view that low visual attention is the primary trigger. Second, the layer with peak energy often marks the onset of the hallucinated preference, while the layer with minimum energy can briefly surface the ground-truth answer before later layers overwrite it. These findings suggest that hallucination is better exposed by the \emph{spatial variation} of visual attention than by its total mass.

Motivated by this insight, we propose \textbf{\sexyname} (\textbf{La}placian-\textbf{S}pectral \textbf{C}ontrastive \textbf{D}ecoding), a training-free decoding strategy for MLLMs. At each decoding step, \sexyname{} scores candidate layers using $E^{(\ell)}$, selects a hallucination-associated peak layer and a ground-truth-associated low-energy layer, and then remaps next-token logits in closed form. The core operation contrasts the final logits against the peak-layer logits, suppressing tokens that are already over-supported at the layer where hallucination starts to form. An optional correction can further use the low-energy layer when validation suggests a benefit. Importantly, \sexyname uses one regular autoregressive forward and only adds lightweight LM-head evaluations and attention-side energy computation, avoiding the second full forward typical of two-view CD pipelines.

We evaluate \sexyname on standard hallucination and general multimodal benchmarks. Across these evaluations, \sexyname achieves state-of-the-art performance: reducing hallucination while preserving general capabilities where most existing methods fail. These results indicate the effectiveness of the proposed structural signal.

Our contributions are summarized as follows:
\begin{itemize}[leftmargin=*]
    \item We identify a layer-wise empirical phenomenon in MLLMs: visual-attention mass can remain high during hallucination, while Laplacian energy over the visual-attention structure reveals both hallucination onset and transient ground-truth recovery.
    \item We introduce \sexyname{}, a training-free decoding method that converts this finding into a simple closed-form logit remapping, contrasting the final prediction against the energy-peak layer and optionally correcting with the low-energy layer.
    \item We validate \sexyname{} across hallucination and general multimodal benchmarks, showing improved faithfulness with modest inference overhead and ablation evidence that supports the proposed visual-attention structure signal.
\end{itemize}

%% file: sections/related_work.tex
\section{Related Work}
\label{sec:related_work}

\noindent\textbf{Multimodal LMs and hallucination.}
Multimodal large language models (MLLMs) have advanced rapidly, building on vision--language pretraining and instruction tuning \cite{alayrac2022flamingo,liu2024visual,bai2023qwen,liu2024llavanext}. Despite strong task performance, they remain prone to \emph{visual hallucination}, producing text that is inconsistent with the image, which has motivated a growing line of analysis and mitigation \cite{yin2023survey,bai2024hallucination}.

\noindent\textbf{Mitigation strategies.}
A first family of methods improves alignment or rewrites model outputs with extra supervision: hallucination-oriented fine-tuning or detection datasets, learned revisors, or preference optimization \cite{gunjal2024detecting,zhouanalyzing,yu2024rlhf}, often at nontrivial data and compute cost.
A complementary line stays \textit{training-free} at test time. \textbf{Attention-intervention} methods directly manipulate attention maps or apply decoding-time penalties~\cite{opera2024cvpr,zhang2024seeing,xing2024mitigating,liu2024paying, iTaD, yin2026dynamic}, which can be effective but may add nontrivial per-step cost when the implementation expands attention computation or re-decoding.
A second \textbf{contrastive-decoding (CD)} family compares logits under two conditions to down-weight spurious continuations, including instruction-style or noised-text variants in language models \cite{chuang2023dola} and visual-input counterparts in MLLMs such as VCD/ICD \cite{vcd2024cvpr,wang2024mitigating,chenhalc,neo2024vord,ASCD,song2026seeing,ren2026nolan}. These approaches are attractive because they do not retrain the backbone, but (i) their contrast can be sensitive to the chosen perturbation, occasionally injecting noise into the next-token distribution, and (ii) \emph{two-view} realizations that require a second end-to-end forward to obtain a second conditional distribution can substantially inflate wall-clock latency relative to a single autoregressive pass. Recently MemVR~\cite{zoulook} and AIR~\cite{zhu2026look} represent an emerging paradigm of \textbf{inference-time visual reinforcement}, which addresses the "visual amnesia" or attention shift towards textual tokens in deeper layers of MLLMs by explicitly re-integrating visual evidence into the language decoder. Specifically, MemVR dynamically detects high-uncertainty layers by entropy and AIR employs \textit{Optimal Transport} to measure the alignment between hidden states and patch embeddings. However, these methods typically rely on coarse-grained statistical metrics to approximate the model’s hallucination state. Such aggregate statistics discard the layer-wise structural dynamics, which are critical for localizing hallucination onset. They overlook the token-level evolutionary trajectories and hierarchical semantic dependencies that are critical for diagnosing and mitigating hallucinations at a deeper mechanistic level.

\noindent\textbf{Positioning of \sexyname.}
\sexyname{} is training-free and operates entirely at decoding time. Unlike magnitude- and entropy-only heuristics, it scores \emph{layer-wise} structure using a high-frequency signal, and then applies a \emph{closed-form} logit remapping. It is related to the CD \emph{principle}: we contrast the final head output against an intermediate one, but the contrasted layer is \emph{selected} on a \emph{single} autoregressive forward, avoiding the two-forward VCD-style pipelines. Relative to layer-wise decoders like DoLa~\cite{chuang2023dola} and DeCo~\cite{wang2025mllm}, our focus is how to find the primary carrier of the hallucination signal. Table~\ref{tab:method_positioning} illustrates the differences of \sexyname{} compared to recent representative SOTA approaches.

%% file: sections/finding.tex
\section{Empirical Analysis}

\subsection{Preliminaries}
\label{sec:prelim}

\textbf{Multimodal large language models.} MLLMs pair a pretrained vision encoder with a decoder-only language model through a lightweight projector. Given an image or video $\mathcal{I}$ and a text prompt $\mathcal{P}$, the vision encoder maps $\mathcal{I}$ to visual tokens $\mathbf{V}\in\mathbb{R}^{N_v\times d}$, which are projected into the language embedding space and concatenated with text tokens $\mathbf{T}\in\mathbb{R}^{N_t\times d}$ to form a single input sequence $\mathbf{X}=[\mathbf{T}_{\text{pre}},\mathbf{V},\mathbf{T}_{\text{post}}]$ of length $N=N_v+N_t$. We denote by $\mathcal{I}_v\subset\{1,\ldots,N_v\}$ the index set of visual tokens in $\mathbf{X}$. At each decoding step, an $L$-layer Transformer decoder consumes $\mathbf{X}$ together with the previously generated tokens. At layer $l$, the last-position (i.e.\ currently generated) query attends to all $N$ keys; we write $\mathbf{A}^{(l,h)}\in\mathbb{R}^{N}$ for the attention distribution of head $h\in\{1,\ldots,H\}$ over key positions, and obtain the next-token logits from the final-layer hidden state via an unembedding head, $\mathbf{z}^{(L)}=\mathrm{LMHead}(\mathbf{h}^{(L)})$, from which the next token is sampled. 

\textbf{Multi-modal Hallucination.} A common narrative in MLLM hallucination is that models hallucinate when they under-attend to the visual evidence, and that increasing or reallocating computation onto visual tokens should therefore reduce hallucination~\cite{opera2024cvpr,liu2024paying,zoulook}. This view implicitly treats the visual slice of the last-query attention as the operational signal. Concretely, we denote the head-averaged attention from the current token to all visual keys $\boldsymbol{\alpha}^{(l)}_{\text{vis}}$, together with its scalar $\ell_1$ mass, by:
\begin{equation}
\label{eq:vis-attn}
\boldsymbol{\alpha}^{(l)}_{\text{vis}}
\;\triangleq\;
\frac{1}{H}\sum_{h=1}^{H}\mathbf{A}^{(l,h)}\!\left[\mathcal{I}_v\right]
\;\in\;
\mathbb{R}^{N_v},
\qquad
m^{(l)}
\;\triangleq\;
\big\|\boldsymbol{\alpha}^{(l)}_{\text{vis}}\big\|_1
\;=\;
\sum_{i\in\mathcal{I}_v}\alpha^{(l)}_{\text{vis},i}
\;\in\;
\mathbb{R}.
\end{equation}
Many prior methods operationalize $\boldsymbol{\alpha}^{(l)}_{\text{vis}}$ in two ways: (i) as a \emph{trigger} that decides \emph{when} to activate a mitigation step (e.g., when $m^{(l)}$ appears small), and/or (ii) as an \emph{edit target}---reallocating, suppressing, or rewiring attention over keys during decoding so that subsequent hidden states and logits are steered toward visually grounded predictions. The empirical analysis below revisits whether such insufficient visual attention is a reliable indicator of hallucination in practice.

\subsection{Is Visual-Attention Magnitude a Reliable Hallucination Signal?}
\label{sec:finding1}

The narrative in Sec.~\ref{sec:prelim} makes a sharp, testable prediction: at the layer that commits to the next token, hallucinated outputs should coincide with \emph{low} visual-attention mass $m^{(l)}$, while confidently correct outputs should coincide with \emph{high} $m^{(l)}$. In this subsection, we revisit this prediction through a representative case (Figure \ref{fig:case}) and quantitative analysis (Figure \ref{fig:case2}). Two observations drawn from this section highlight the limitations of relying solely on the magnitude-based account.

\begin{figure}
    \centering
    \includegraphics[width=1.0\linewidth]{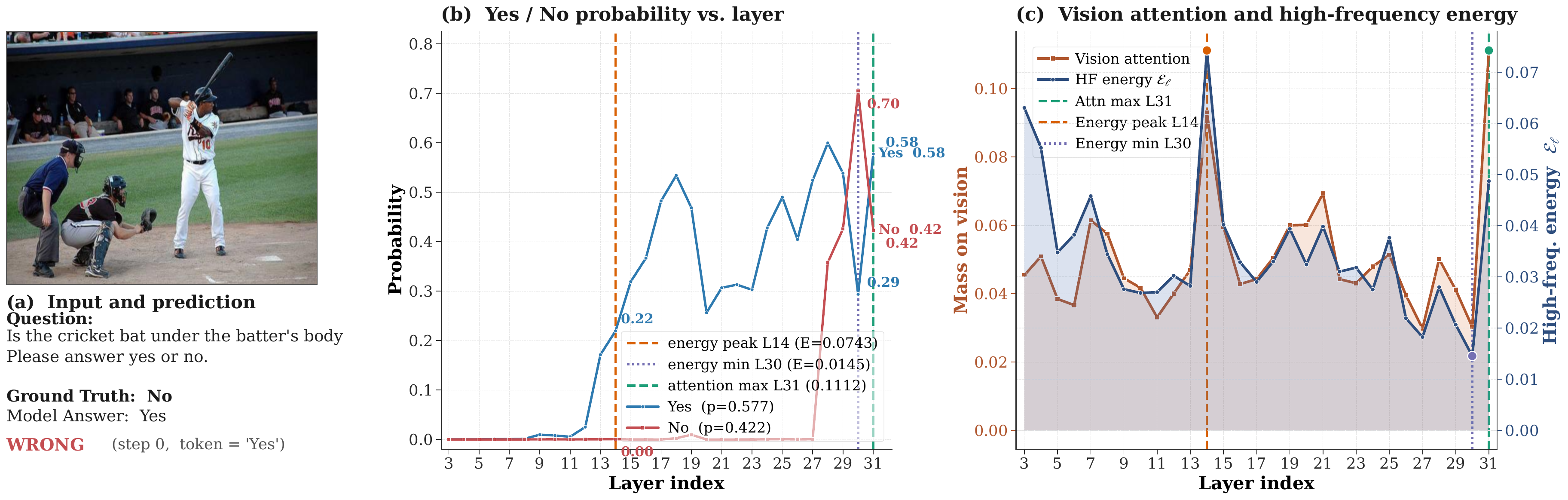}
    \caption{\textbf{Three views of the same decoding step under confident hallucination.}
    LLaVA-1.5 is asked \textit{``Is the cricket bat under the batter's body?''}; ground truth is \emph{No}, but the model answers \emph{Yes} with $p(\text{Yes})\!=\!0.58$. Using the \emph{same} attention patterns, for each layer we plot:
    \textbf{(middle)} top-$2$ probabilities of \emph{Yes} and \emph{No} from the hidden state at that layer;
    \textbf{(right)} visual attention mass $m^{(l)}$ from Eq.~\ref{eq:vis-attn} and high-frequency energy $E^{(l)}$ of $\boldsymbol{\alpha}^{(l)}_{\text{vis}}$ from Eq.~\ref{eq:energy}.
    The layer that maximizes $E^{(l)}$ tends to be the first at which the hallucinated token reaches top-$1$ and begins a monotonic rise in probability; the layer that minimizes $E^{(l)}$ tends to favor the ground-truth token and is often the only layer at which the correct logit exceeds its hallucinated competitor. More cases is available at Appendix~\ref{apx:case_study}.}
    \label{fig:case}
\end{figure}

\textbf{Confident hallucination coincides with high visual attention magnitude.}
Two observations drawn from Fig.~\ref{fig:case} jointly falsify the magnitude-based account for this instance. First, the hallucination does not arise as a last-layer artifact: the logit of \emph{Yes} surpasses that of \emph{No} in the intermediate layers and reaches a large margin well prior to the final output, indicating that the model's commitment to the erroneous response is established and sustained across a substantial portion of the forward pass. Second, $m^{(l)}$ remains substantially high across the shallow and intermediate layers throughout the onset of hallucination, and ultimately attains its global maximum precisely at the layer responsible for emitting the hallucinated token.

Taken together, these observations reveal an important suboptimality in the current paradigm: the model maintains robust visual engagement during the onset of the hallucination and allocates maximal $m^{(l)}$ at the very layer where its confidence in the incorrect answer is highest. This result stands in direct opposition to the prediction of the magnitude hypothesis, which implicitly assumes that greater visual engagement should correspond to more faithful grounding.

\begin{center}
\vspace{-0.7em}
\begin{tcolorbox}[
    colback=violet!8!white,
    colframe=violet!70!black,
    width=1.0\textwidth, 
    boxrule=0.5mm, 
    arc=5mm, 
    auto outer arc,
    left=0.5mm,
    right=0.5mm
]
\centering
\textbf{\textit{Observation 1. Visual attention magnitude often remains high during hallucination onset, making it a suboptimal proxy for exact error localization.}}
\end{tcolorbox}
\vspace{-0.7em}
\end{center}

\subsection{Energy-Peak Layer Pre-Commits the Final Answer}
\label{sec:finding2}
While attention magnitude may lack the fine-grained resolution to isolate these internal dynamics, it is natural to investigate alternative properties of the spatial distribution. Beyond \emph{how much} mass is allocated to the visual segment, we examine \emph{how} this mass is spatially distributed across the $N_v$ visual tokens, specifically whether the pattern is smooth or exhibits sharp, position-wise variations.

Recent findings in unimodal generation~\cite{qi2026detecting} indicate that erroneous outputs often co-occur with pronounced local fluctuations in attention, such as sharp peaks and abrupt shifts that global statistics fail to capture. Inspired by this insight, we investigate whether analogous \emph{spatial} structures within the visual attention slice of a multimodal decoder can reveal the model's layer-wise commitment during the forward pass. Accordingly, we introduce a metric that explicitly quantifies these local variations. We summarize this behavior using a single layer-wise scalar: the \emph{high-frequency energy} of the visual attention. Concretely, let $\boldsymbol{\alpha}^{(l,h)}_{\text{vis}} \triangleq \mathbf{A}^{(l,h)}\!\left[\mathcal{I}_v\right] \in \mathbb{R}^{N_v}$ denote the per-head visual slice of the last-query attention. Assuming the vision encoder produces tokens on a square grid such that $N_v = G_v^2$, we reshape this per-head slice into a 2D map:

\begin{equation}
\label{eq:reshape}
\mathbf{A}^{(l,h)}_{\text{2D}} \;\triangleq\; \mathrm{reshape}_{G_v\times G_v}\!\big(\boldsymbol{\alpha}^{(l,h)}_{\text{vis}}\big) \;\in\; \mathbb{R}^{G_v\times G_v}.
\end{equation}

The high-frequency energy is then defined as the head-averaged Frobenius norm of the 2-D discrete Laplacian response:
\begin{equation}
\label{eq:energy}
E^{(l)} \;=\; \frac{1}{H}\sum_{h=1}^{H} \Big\| \mathcal{K}_{\!\Delta} \,\ast\, \mathbf{A}^{(l,h)}_{\text{2D}} \Big\|_F,
\end{equation}

where $\ast$ denotes 2D convolution with zero padding and $\mathcal{K}_{\!\Delta}$ denotes the Laplace operator. For non-square layouts we apply the one-dimensional second-difference operator $\big[\mathcal{K}_{\!\Delta}^{\,\text{1D}}\,\mathbf{x}\big]_i = x_{i+1} - 2\,x_i + x_{i-1}$ along the token axis. We also evaluated Sobel-style first differences and Laplacian-of-Gaussian filters; all choices reproduced the qualitative behavior in Fig.~\ref{fig:case}.

\begin{wrapfigure}{R}{0.5\textwidth}
    \centering
    \includegraphics[width=\linewidth]{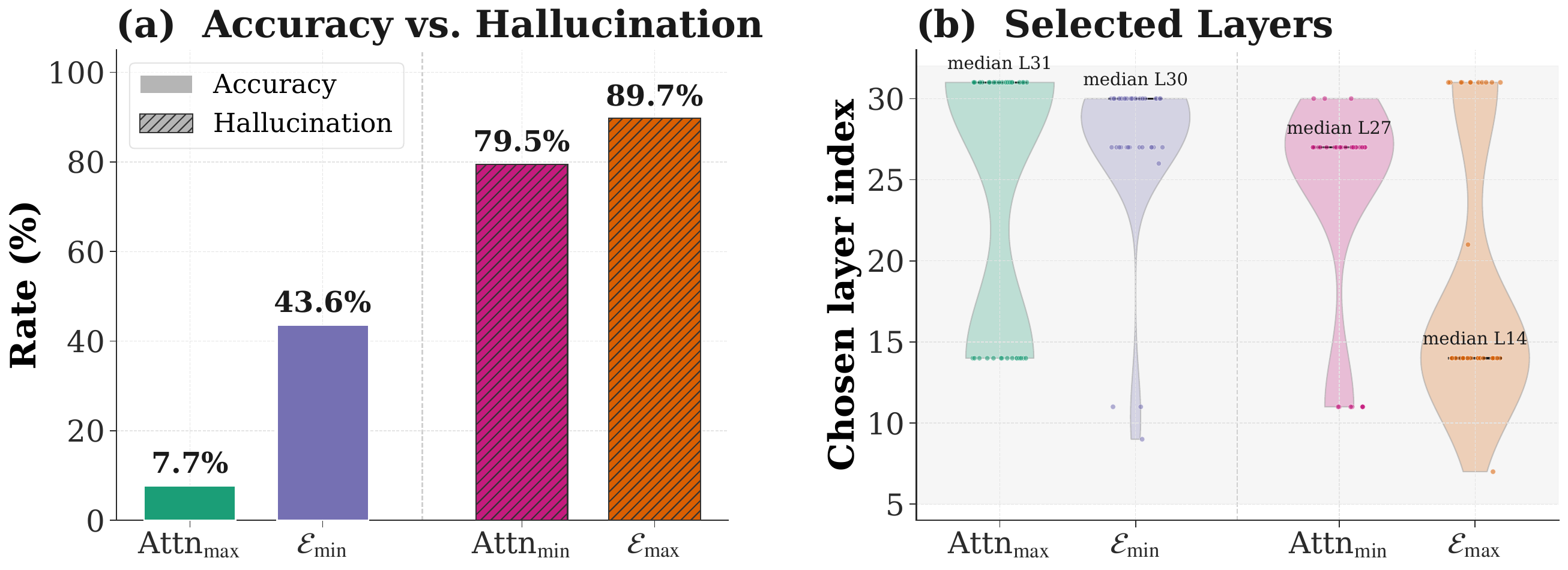}
    \caption{Quantitative analysis on the MME-Hall subset when MLLM exhibits hallucination. The layer with the largest attention mass attains only 7.7\% answer accuracy, whereas the layer with the minimum high-frequency energy recovers the correct response in 43\% of cases. Moreover, the layer with the maximum high-frequency energy shows a 10 percentage-point higher hallucination rate than the layer with the low visual attention.}
    \label{fig:case2}
\end{wrapfigure}

\textbf{Energy-peak layer locates hallucination; energy-minimum layer recovers GT.}
Figure~\ref{fig:case} exhibits two consistent cross-panel patterns. First, $E^{(l)}$ attains its maximum at an intermediate layer at which the hallucinated token \emph{Yes} has already begun to rise; from that layer onward, $p(\text{Yes})$ increases monotonically toward the final output. Second, $E^{(l)}$ reaches its minimum at a deeper layer, and at that layer alone the correct token \emph{No} regains the top-$1$ rank over \emph{Yes}. Thus two summaries of the \emph{same} visual attention carry contrasting signals: high mass coincides with strong visual engagement even when the model is wrong, whereas high-frequency structure marks an earlier internal pre-commitment to the hallucinated token and a late, transient preference for the ground-truth answer that deeper layers subsequently override.

\begin{center}
\vspace{-0.7em}
\begin{tcolorbox}[
    colback=violet!8!white,
    colframe=violet!70!black,
    width=1.0\textwidth, 
    boxrule=0.5mm, 
    arc=5mm, 
    auto outer arc,
    left=0.5mm,
    right=0.5mm
]
\centering
\textbf{\textit{Observation 2. $E^{(l)}$ effectively localizes internal commitments missed by magnitude, where its peak marks the hallucination and its minimum reveals the ground truth.}}
\end{tcolorbox}
\vspace{-0.7em}
\end{center}

%% file: sections/method.tex
\section{Methodology}
\label{sec:method}

\subsection{Overview}
\label{sec:method-bridge}

Sec.~\ref{sec:finding1}--\ref{sec:finding2} argue that visual-attention \emph{magnitude} can be a suboptimal proxy for localizing hallucination, whereas the Laplacian energy $E^{(\ell)}$ tracks the hallucination and the ground truth.

Motivated by these observations, we propose \textbf{\sexyname} (\textbf{La}placian-\textbf{S}pectral \textbf{C}ontrastive \textbf{D}ecoding), a training-free decoding strategy. \sexyname scores layers with $E^{(\ell)}$ (Sec.~\ref{sec:method-energy}), then remaps next-token logits: (i)~contrastive cancellation of the peak-layer preference relative to $\mathbf{z}^{(L)}$ (Sec.~\ref{sec:method-alpha}); (ii)~a confidence-guided correction from $\ell_{\mathrm{gt}}$ (Sec.~\ref{sec:method-beta}).
Section~\ref{sec:method-summary} gives the closed-form logits.

\subsection{Layer Selection}
\label{sec:method-energy}

We first evaluated $E^{(\ell)}$ on a chosen set of decoder layers $\mathcal{L}$.
At each step we take the last-query row of $\mathbf{A}^{(\ell,h)}$, restrict to visual keys when analyzing vision tokens, map the slice to a grid when applicable, and compute $E^{(\ell)}$ via the Laplacian response (Eq.~\ref{eq:energy}); alternative spatial layouts follow the same principle as in the empirical section.
We then select
\begin{equation}
\label{eq:peak-layers}
\ell_{\mathrm{peak}} \;\in\; \arg\max_{\ell\in\mathcal{L}} E^{(\ell)},
\qquad
\ell_{\mathrm{gt}} \;\in\; \arg\min_{\ell\in\mathcal{L}} E^{(\ell)},
\end{equation}
during the forward process. $\ell_{\mathrm{peak}}$ targets the onset of high-curvature, hallucination-associated attention structure, while $\ell_{\mathrm{gt}}$ targets a depth that often briefly aligns with the correct token.

\subsection{Contrastive Logits (\sexyname-$\alpha$)}
\label{sec:method-alpha}

Intermediate logits at $\ell_{\mathrm{peak}}$ tend to lock in the eventual error with a large margin, which the final layer inherits; we therefore contrast the final logits against that depth while keeping $\mathbf{z}^{(L)}$ as the anchor:
\begin{equation}
\label{eq:lascd-alpha}
\tilde{\mathbf{z}}
\;=\;
(1+\alpha)\,\mathbf{z}^{(L)}
\;-\;
\alpha\,\mathbf{z}^{(\ell_{\mathrm{peak}})},
\qquad
\alpha \ge 0,
\end{equation}
where $\mathbf{z}^{(\ell_{\mathrm{peak}})}$ is the LM-head output at $\ell_{\mathrm{peak}}$.
When $\mathbf{z}^{(\ell_{\mathrm{peak}})}\approx \mathbf{z}^{(L)}$, Eq.~\ref{eq:lascd-alpha} leaves benign trajectories largely unchanged.
We apply top-$K$ masking on the \emph{composed} logits, following standard contrastive-decoding practice.

\subsection{Decoding Correction (\sexyname-$\beta$)}
\label{sec:method-beta}

The minimum-energy depth $\ell_{\mathrm{gt}}$ can carry a complementary preference aligned with the ground-truth token (Sec.~\ref{sec:finding2}).
To add a small correction $\propto \mathbf{z}^{(\ell_{\mathrm{gt}})}$ without letting fixed scaling amplify mismatched logit scales across depth, we gate the strength by the peak softmax mass of that layer's distribution:
\begin{equation}
\label{eq:gt-conf}
\pi_{\max}^{(\ell)} \;\triangleq\; \max_{v\in\mathcal{V}} \big[\mathrm{softmax}(\mathbf{z}^{(\ell)})\big]_v,
\qquad
\beta_{\mathrm{eff}} \;=\; \beta\cdot\pi_{\max}^{(\ell_{\mathrm{gt}})},
\quad
\beta \ge 0.
\end{equation}

\subsection{Closed-Form of \sexyname}
\label{sec:method-summary}

\sexyname-$\alpha$ and \sexyname-$\beta$ combine additively:
\begin{equation}
\label{eq:lascd-full}
\tilde{\mathbf{z}}
\;=\;
(1+\alpha)\,\mathbf{z}^{(L)}
\;-\;
\alpha\,\mathbf{z}^{(\ell_{\mathrm{peak}})}
\;+\;
\beta_{\mathrm{eff}}\,\mathbf{z}^{(\ell_{\mathrm{gt}})}.
\end{equation}

In practice, we found that contrastive decoding with $\alpha$ alone (Eq.~\ref{eq:lascd-alpha}) is already sufficient across several benchmarks. Unless otherwise stated, \textbf{all main experiments and reported results use $\beta{=}0$}. We enable $\beta{>}0$ only when validation on a specific setting suggests a benefit.

\textbf{Complexity Analysis.}
At each autoregressive step we use \emph{one} forward~$F$ to compute $E^{(\ell)}$ on~$\mathcal{L}$, pick $\ell_{\mathrm{peak}}$ and~$\ell_{\mathrm{gt}}$, and form $\tilde{\mathbf{z}}$ in Eq.~\ref{eq:lascd-full}; the overhead beyond~$F$ is $\mathcal{O}(|\mathcal{L}|\,H\,N_v)$ for the Laplacian energy and $\mathcal{O}(|\mathcal{V}|)$ for the logit mix. By contrast, \emph{existing} methods in Table~\ref{tab:la-complexity} can incur (i) a \emph{second} full forward ($2F$) when the scheme requires a second pass through the stack, (ii) a \emph{per-probed-layer} cost that multiplies~$|\mathcal{L}|$ with a full $|\mathcal{V}|$-dimensional logit/entropy readout, or (iii) a post-$F$ re-weighting that scales with model hidden dimension as~$\mathcal{O}(N_v d)$. \sexyname avoids the duplicated backbone pass, keeps a \emph{single}~$F$, and confines the nontrivial work to the attention--vision slice in $\mathcal{O}(|\mathcal{L}|\,H\,N_v)$ with only a \emph{bounded} $\mathcal{O}(d|\mathcal{V}|)$ logit merge in Eq.~\ref{eq:lascd-full}, which sidesteps the $d|\mathcal{L}||\mathcal{V}|$ and $N_v d$ patterns above and is typically more favorable in large-vocab settings.

\begin{table}[htp]
\centering
\caption{Complexity comparison of different methods; $F$: one autoregressive forward, $H$: number of attention heads, $N_v$: visual positions, $d$: hidden dim, $\mathcal{L}$: layer subsets, and $|\mathcal{V}|$: the vocabulary size.}
\label{tab:la-complexity}
\resizebox{1.0 \linewidth}{!}{
\begin{tabular}{ccccc}
\toprule
 & VCD & DeCo & MemVR & \sexyname (ours) \\
\midrule
Complexity
  & $\mathcal{O}(2F+d|\mathcal{V}|)$
  & $\mathcal{O}(F+d|\mathcal{L}|\,|\mathcal{V}|)$
  & $\mathcal{O}(F+d|\mathcal{L}|\,|\mathcal{V}|+N_v d)$
  & $\mathcal{O}(F+|\mathcal{L}|\,H\,N_v+d|\mathcal{V}|)$ \\
\bottomrule
\end{tabular}
}
\end{table}

\textbf{Theoretical Analysis.}
We provide a lightweight toy geometric justification for using Laplacian energy as a
layer-wise signal. For convenience, we stated a restricted setting and analyzed the \emph{normalized} visual-attention shape $\boldsymbol{\alpha}^{(\ell)}$ with $\mathbf{1}^{\top}\boldsymbol{\alpha}^{(\ell)}=1$, so the total visual-attention mass is fixed and only the spatial structure is studied. This toy analysis explains why spatially fragmented visual attention should carry larger high-frequency energy. The full proof is in Appendix~\ref{apx:theory}.

Consider where visual attention is normalized over visual tokens, $\boldsymbol{\alpha}\in\mathbb{R}_{\ge0}^{N_v}$ and
$\mathbf{1}^{\top}\boldsymbol{\alpha}=1$, and where grounding failure is modeled locally on the probability simplex as
$\boldsymbol{\alpha}_{\mathrm{hall}}=\boldsymbol{\pi}+\boldsymbol{\eta}$ with $\mathbf{1}^{\top}\boldsymbol{\eta}=0$, $\mathbb{E}[\boldsymbol{\eta}]=0$, and non-trivial
high-pass variance
$\operatorname{Tr}(\Delta^\top\Delta\,\mathrm{Cov}(\boldsymbol{\eta}))>0$.

\begin{theorem}
\label{thm:normalized-energy-summary}
Under the fixed-mass condition, coherent region-localized attention has bounded Laplacian energy, while fragmented or spatially incoherent attention increases the expected
squared Laplacian energy relative to the unperturbed baseline:
\begin{equation}
    \mathbb{E}\!\left[E(\boldsymbol{\alpha}_{\mathrm{hall}})^2\right]
    >
    E(\boldsymbol{\pi})^2.
    \label{eq:energy-amplification-sum}
\end{equation}
\end{theorem}

\noindent\textbf{Remark.}
This toy analysis provides a geometric interpretation of $E^{(\ell)}$: low Laplacian energy corresponds to spatially coherent and region-localized attention, while higher energy is consistent with fragmented or incoherent visual grounding.


%% file: sections/experiment.tex
\section{Experiment}

\subsection{Experiment Setup}
\textbf{Baselines and models.}
We evaluate our method on four representative MLLMs for image--language understanding: InstructBLIP~\cite{li2023blip}, LLaVA-1.5-7B~\cite{liu2023improved}, Qwen-VL-Chat~\cite{bai2023qwen}, and GLM-4V-9B~\cite{glm2024chatglm}.
For video--language understanding, we additionally report results on Video-LLaVA-7B~\cite{lin-etal-2024-video} and LLaVA-Next-Video-7B-DPO~\cite{liu2024llavanext}.
For comparison, we include several strong training-free hallucination mitigation baselines, including OPERA~\cite{opera2024cvpr}, DoLa~\cite{chuang2023dola}, VCD~\cite{vcd2024cvpr}, ICD~\cite{wang2024mitigating}, SID~\cite{huo2024self}, DeCo~\cite{wang2025mllm}, MemVR~\cite{zoulook}, ASCD~\cite{ASCD}, NoLan~\cite{ren2026nolan}, AIR~\cite{zhu2026look}, DMAS~\cite{yin2026dynamic} and Self-PEP~\cite{videohallucer}.

\textbf{Datasets and Evaluation.} 
To rigorously assess the effectiveness of our proposed method, we conduct a comprehensive set of experiments across POPE benchmark~\cite{li2023evaluating}, CHAIR~\cite{rohrbach2018object}, MME~\cite{fu2023mme}, MM-Vet~\cite{2024MMVet}, LLaVA-Bench (in-the-wild)~\cite{liu2024visual}, HallusionBench~\cite{guan2024hallusionbench} and VideoHallucer benchmark~\cite{videohallucer}.

\textbf{Implementation Details.}
Usually, we set $\alpha=0.1$ and $\mathcal{L}=\{8,9,\ldots,29\}$. We set $\beta \in \{0.2, 0.6\}$ in HallusionBench~\cite{guan2024hallusionbench} and CHAIR~\cite{rohrbach2018object} accordingly. All settings of baseline methods follow the default configurations from the original papers. More details are in Appendix~\ref{apx:experimental_details}.

\subsection{Evaluations on Hallucination Benchmarks}

\input{tables/chair}

We evaluate hallucination on CHAIR, POPE, and HallusionBench, and further test generalization on video understanding; results are in Table~\ref{tab:CHAIR}, Table~\ref{tab:res_pope}, Table~\ref{tab:HallusionBench} and Table~\ref{tab:video-hallucination-main}. In the POPE evaluation, \sexyname is effective across two backbones, clearly surpassing prior SOTAs, with an average accuracy increase of up to +7.9\% and an F1-score increase of up to +7.3\% under the Random, Popular, and Adversarial settings. On CHAIR dataset, \sexyname improves performance by 19.8\% and 34.7\% over vanilla LLaVA-1.5 and Qwen-VL-Chat. Table \ref{tab:HallusionBench} shows the HallusionBench evaluation results, where \sexyname consistently yields the best results among all comparators, including hardaACC and aACC. On video hallucination, \sexyname generalizes well and outperforms baselines by large margins, demonstrating strong robustness beyond image understanding benchmarks.

\input{tables/POPE}

\input{tables/HallusionBench}

\input{tables/Videohallucer}

\subsection{Evaluations on General Capability Benchmarks}

\input{tables/MME}

We evaluate \sexyname on general-capability benchmarks, including MME, MM-Vet, and LLaVA-Bench. We tested three different MLLM backbones. On LLaVA-Bench (Table~\ref{tab:res_llavabench}), \sexyname consistently surpasses the compared baselines. On MME subsets that target object- and attribute-level hallucination together with perception- and cognition-oriented items (Table \ref{tab:MME}), \sexyname delivers uniform improvements at both granularities: the aggregate MME-Hall score increases by 26 and 21 point for the LLaVA- and Qwen-VL–based models, respectively. By contrast, many prior mitigation pipelines sacrifice generality. For example, DeCo causes 89 and 175 point drop on overall perception score for LLaVA-1.5 and GLM-4V. The hallucination performance is improved, yet scores on these broad benchmarks fall. Relative to other CD-based alternatives, \sexyname pairs strong hallucination control with reliable performance on general-purpose evaluation.

\subsection{Method Analysis}
\input{tables/ablation}

\textbf{Ablation Studies.} Table~\ref{tab:CHAIR_ab} verifies that each design choice in \sexyname contributes to hallucination reduction while largely preserving answer coverage. Starting from LLaVA-1.5, LaSCD-$\alpha$ already lowers the average CHAIR score from 32.7 to 29.7 and slightly improves recall, showing that subtracting the hallucination-prone peak-layer preference is an effective first-order correction. Adding the $\beta$ correction further reduces hallucination, but its lower recall suggests that directly injecting the low-energy layer can become conservative when used alone. Using the Sobel operator $\mathcal{K}_{\nabla}$ further improves CHAIR$_I$ to 11.7, suggesting that first-order edge responses can also capture hallucination-related attention irregularities; however, the full Laplacian-based design still gives the best overall average score.

We also compare alternative layer-scoring signals and layer ranges. Using 1D Laplacian kernel $\mathcal{K}_{1D}$, text-side energy $E_{\mathrm{txt}}$, raw visual mass $m^{(\ell)}$, or Shannon entropy $\mathrm{H}_{\text{Shannon}}$ improves over the baseline, but all are weaker than the full method, supporting our claim that the high-frequency energy is a more reliable hallucination indicator than magnitude or dispersion alone.
Likewise, using only a fixed layer $\mathcal{L}{=}\{20\}$ or all layers $\mathcal{L}{=}\{1,\ldots,L\}$ underperforms the final design, suggesting that hallucination-related evidence is layer-dependent and benefits from targeted layer selection.
Overall, the full \sexyname achieves the best average CHAIR score while keeping response length and recall close to the original model. This indicates that the improvement is not obtained by simply shortening answers or suppressing object mentions, but by more accurately reducing hallucinated content.

\textbf{Inference Latency.} \sexyname introduces only a small amount of \emph{additional} compute on top of the \emph{single} autoregressive forward: only two extra decoding head computation to form the contrast decoding. Figure~\ref{fig:radar} summarizes measured per-token latency, showing that \sexyname remains close to the original model while two-forward CD-style baselines incur substantially higher step time.

%% file: tables/chair.tex
\begin{wraptable}{R}{0.5\textwidth}
\centering
\vspace{-1em}
\caption{CHAIR evaluation results of different methods. The max output token length is set to 1024. The best results are in \sethlcolor{lightpink}\hl{green}.}
\resizebox{1.0 \linewidth}{!}{
\begin{tabular}{lccccc}
\toprule
 Methods  & CHAIR$_S$ $\downarrow$ & CHAIR$_I$ $\downarrow$ & {Average} $\downarrow$ & {Len} & {Recall} $\uparrow$ \\ 
\midrule
LLaVA-1.5  & 50.0 \basexdown{0.0} & 15.4 \basexdown{0.0} & 32.7 \basexdown{0.0} & 100.6 & 77.1 \basex{0.0} \\
\ + OPERA & {47.8} \down{2.2} & {14.6} \down{0.8} & {31.2} \down{0.5} & 98.6 & 76.8 \downbad{0.3}  \\
\ + VCD & 48.6 \down{1.4} & 14.9 \down{0.5} & 31.8 \down{0.5} & 100.4 & {77.3} \up{0.2}  \\
\ + ICD & 56.2 \upbad{6.2} & 16.3 \upbad{0.9} & 36.3 \upbad{3.6} & 103.4 & 16.3 \downbad{60.} \\
\ + DeCo & 42.8 \down{7.2} & 12.6 \down{2.8} & 27.7 \down{5.0} & 103.4 & 76.7 \downbad{0.4} \\
\ + MemVR & 47.6 \down{2.4} & 13.8 \down{1.6} & 30.7 \down{2.0} & 101.5 & 79.6 \up{2.5}  \\ 
\ \textbf{+ Ours} & \sethlcolor{lightpink}\hl{40.6} \down{9.4} & \sethlcolor{lightpink}\hl{11.8} \down{2.4} & \sethlcolor{lightpink}\hl{26.2} \down{6.5} & 100.0 & 77.6 \up{0.5}  \\ 

\midrule
Qwen-VL & 6.80 \basex{0.0}  & 5.30 \basex{0.0}  & 6.05 \basex{0.0} & 17.6 & 53.4 \basex{0.0} \\
 + OPERA  & - & - & - & - & -\\
 + VCD & {13.0} \upbad{6.2} & 12.3 \upbad{7.0} & 12.6 \upbad{6.5} & 115.  & 47.9 \downbad{5.5} \\
 + ICD & 18.4 \upbad{11.} & 14.3 \upbad{9.0} & 16.3 \upbad{10.} & 48.1 & 37.6 \downbad{15.}\\
 + DeCo & 12.2 \upbad{5.4} & 6.40 \upbad{1.1} & 9.30 \upbad{3.2} & 31.2 & 57.3 \up{3.9}\\
  {+ MemVR}  & {4.80} \down{2.0} & {3.30} \down{2.0} & {4.05} \down{2.0} & 15.0 & 52.3 \downbad{1.1} \\
  \textbf{+ Ours}  & \sethlcolor{lightpink}\hl{4.80} \down{2.0} & \sethlcolor{lightpink}\hl{3.00} \down{2.3} & \sethlcolor{lightpink}\hl{3.90} \down{2.1} & 16.6 & 52.3 \downbad{1.1} \\
\bottomrule
\end{tabular}
}
\label{tab:CHAIR}
\vspace{-1em}
\end{wraptable}

%% file: tables/POPE.tex
\begin{table}[t]
\centering
\caption{Performance evaluation on POPE benchmark. The best results are in \sethlcolor{lightpink}\hl{green}. We report accuracy and f1-score under three settings, e.g., \emph{Random}, \emph{Popular}, \emph{Adversarial}, and \emph{Average}, to show the robustness of the different methods directly. } 
\resizebox{1.0\linewidth}{!}{
\begin{tabular}{llcccccccc}
\toprule
\multirow{2}{*}[-0.5ex]{{Evaluation}} 
 & \multirow{2}{*}[-0.5ex]{{Methods}} 
 & \multicolumn{2}{c}{\;Random}  
 & \multicolumn{2}{c}{\;Popular} 
 & \multicolumn{2}{c}{\;Adversarial} 
 & \multicolumn{2}{c}{\;Average} \\ 
\cmidrule(lr){3-4} \cmidrule(lr){5-6} \cmidrule(lr){7-8} \cmidrule(l){9-10}
&  
& Accuracy $\uparrow$ & F1-score $\uparrow$ 
& Accuracy $\uparrow$ & F1-score $\uparrow$  
& Accuracy $\uparrow$ & F1-score $\uparrow$  
& Accuracy $\uparrow$ & F1-score $\uparrow$  
\\ 
\midrule
\multirow{10}{*}[-4ex]{{LLaVA-1.5}}
 & {Regular}
 & 83.49 \basex{0.0} & 82.28 \basex{0.0} 
 & 79.98 \basex{0.0} & 79.34 \basex{0.0} 
 & 76.03 \basex{0.0} & 76.26 \basex{0.0} 
 & 79.83 \basex{0.0} & 79.29 \basex{0.0} 
\\

& OPERA \citep{opera2024cvpr}
 & {87.53} \basexx{4.0} & 86.45 \basexx{4.2} 
 & {84.21} \basexx{4.2} & {83.50} \basexx{4.2}
 & 80.88 \basexx{4.9} & {80.69} \basexx{4.4}
 & {84.21} \basexx{4.4} & {83.55} \basexx{4.3} 
\\

& DoLa \citep{chuang2023dola}
 & 85.78 \basexx{2.2} & 84.69 \basexx{2.4} 
 & 80.75 \basexx{0.7} & 81.11 \basexx{1.7} 
 & 77.32 \basexx{1.3} & 76.66 \basexx{0.4} 
 & 81.28 \basexx{1.4} & 80.82 \basexx{1.5} 
\\

& VCD \citep{vcd2024cvpr}
 & 86.84 \basexx{3.4} & {86.83} \basexx{4.6} 
 & 82.65 \basexx{2.7} & 83.37 \basexx{4.0} 
 & 77.31 \basexx{1.3} & 79.28 \basexx{3.0} 
 & 82.27 \basexx{2.4} & 83.16 \basexx{3.9} 
\\

& ICD \citep{wang2024mitigating}
 & 84.87 \basexx{1.4} & 83.27 \basexx{1.0} 
 & 82.93 \basexx{3.0} & 81.45 \basexx{2.1} 
 & {81.07} \basexx{5.0} & 79.96 \basexx{3.7} 
 & 82.96 \basexx{3.1} & 81.56 \basexx{2.3} 
\\

& SID \citep{huo2024self}
 & 87.86 \basexx{4.4} & 85.73 \basexx{3.4} 
 & 83.79 \basexx{3.8} & 85.42 \basexx{6.1} 
 & 82.54 \basexx{6.5} & 81.98 \basexx{5.7} 
 & 84.73 \basexx{4.9} & 84.37 \basexx{5.0} 
\\

& DeCo \citep{wang2025mllm}
 & 86.66 \basexx{3.2} & 89.15 \basexx{6.9} 
 & 84.63 \basexx{4.7} & 85.84 \basexx{6.5} 
 & 77.63 \basexx{1.6} & 80.64 \basexx{4.4} 
 & 82.97 \basexx{3.1} & 85.21 \basexx{5.9} 
\\

& {MemVR} \citep{zoulook}
 & 88.50 \basexx{5.0}  
 & 87.34 \basexx{5.0}
 & 86.07 \basexx{6.1}  
 & 85.60 \basexx{6.2}
 & 81.97 \basexx{5.9}  
 & 81.33 \basexx{5.0}
 & 85.51 \basexx{5.6}  
 & 84.75 \basexx{5.4}
\\ 

& {ASCD} \citep{ASCD}
 & 89.48 \basexx{6.0} & 89.09 \basexx{6.8} 
 & 87.20 \basexx{7.2} & 86.69 \basexx{7.4} 
 & 82.90 \basexx{6.9} & 82.97 \basexx{6.7} 
 & 86.53 \basexx{6.7} & 86.25 \basexx{7.0} 
\\

& {NoLan} \citep{ren2026nolan}
 & 87.11 \basexx{3.6} & 86.60 \basexx{4.3} 
 & 85.81 \basexx{5.8} & 85.17 \basexx{5.8} 
 & 83.83 \basexx{7.8} & 82.58 \basexx{6.3} 
 & 85.58 \basexx{5.8} & 84.78 \basexx{5.5} 
\\

& {AIR} \citep{zhu2026look}
 & 89.00 \basexx{5.5} & 88.20 \basexx{5.9} 
 & 87.10 \basexx{7.1} & 86.40 \basexx{7.1} 
 & \sethlcolor{lightpink}\hl{83.90} \basexx{7.9} 
 & \sethlcolor{lightpink}\hl{83.60} \basexx{7.3} 
 & 86.67 \basexx{6.8} & 86.07 \basexx{6.8} 
\\

& {DMAS} \citep{yin2026dynamic}
 & 90.03 \basexx{6.5} & 90.02 \basexx{7.7} 
 & 87.33 \basexx{7.4} & 87.03 \basexx{7.7} 
 & 83.07 \basexx{7.0} & 83.33 \basexx{7.1} 
 & 86.81 \basexx{7.0} & 86.79 \basexx{7.5} 
\\

& {\textbf{Ours}} 
 & \sethlcolor{lightpink}\hl{90.57} \basexx{7.0}  
 & \sethlcolor{lightpink}\hl{90.86} \basexx{8.0}
 & \sethlcolor{lightpink}\hl{87.45} \basexx{7.5}  
 & \sethlcolor{lightpink}\hl{87.40} \basexx{8.0}
 & 83.46 \basexx{7.4}  
 & 82.60 \basexx{6.3}
 & \sethlcolor{lightpink}\hl{87.16} \basexx{7.3}  
 & \sethlcolor{lightpink}\hl{86.95} \basexx{7.6}
\\ 

\midrule

\multirow{4}{*}[-4ex]{{InstructBLIP}}
 & {Regular}
 & 80.42 \basex{0.0} & 80.94 \basex{0.0} 
 & 76.09 \basex{0.0} & 77.65 \basex{0.0} 
 & 72.37 \basex{0.0} & 75.42 \basex{0.0} 
 & 76.29 \basex{0.0} & 78.00 \basex{0.0} 
\\

& {OPERA} \citep{opera2024cvpr}
 & 84.57 \basex{4.2} & 84.69 \basex{3.8} 
 & 78.52 \basex{2.4} & 81.02 \basex{3.4} 
 & 75.60 \basex{3.2} & 78.37 \basex{3.0} 
 & 79.56 \basex{3.3} & 81.36 \basex{3.4} 
\\

& {DoLa} \citep{chuang2023dola}
 & 83.01 \basex{2.6} & 83.04 \basex{2.1} 
 & 78.97 \basex{2.9} & 79.83 \basex{2.2} 
 & 74.76 \basex{2.4} & 76.58 \basex{1.2} 
 & 78.91 \basex{2.6} & 79.82 \basex{1.8} 
\\

& {VCD} \citep{vcd2024cvpr}
 & 84.11 \basex{3.7} & 84.13 \basex{3.2} 
 & 79.64 \basex{3.6} & 80.80 \basex{3.2} 
 & 76.42 \basex{4.1} & 78.03 \basex{2.6} 
 & 80.06 \basex{3.8} & 80.99 \basex{3.0} 
\\

& {ICD} \citep{wang2024mitigating}
 & 82.57 \basex{2.2} & 83.32 \basex{2.4} 
 & 78.23 \basex{2.1} & 79.12 \basex{1.5} 
 & 75.28 \basex{2.9} & 76.24 \basex{0.8} 
 & 78.69 \basex{2.4} & 79.56 \basex{1.6} 
\\

& {SID} \citep{huo2024self}
 & 86.56 \basex{6.1} & 85.94 \basex{5.0} 
 & 80.26 \basex{4.2} & 81.75 \basex{4.1} 
 & 77.64 \basex{5.3} & 80.41 \basex{5.0} 
 & 81.49 \basex{5.2} & 82.70 \basex{4.7} 
\\

& \textbf{Ours}
 & \sethlcolor{lightpink}\hl{90.02} \basex{9.6} 
 & \sethlcolor{lightpink}\hl{89.75} \basex{8.8} 
 & \sethlcolor{lightpink}\hl{83.56} \basex{7.4} 
 & \sethlcolor{lightpink}\hl{83.87} \basex{6.2} 
 & \sethlcolor{lightpink}\hl{81.03} \basex{8.6} 
 & \sethlcolor{lightpink}\hl{81.83} \basex{6.4} 
 & \sethlcolor{lightpink}\hl{84.87} \basex{8.5} 
 & \sethlcolor{lightpink}\hl{85.15} \basex{7.1} 
\\
\bottomrule
\end{tabular}
}
\label{tab:res_pope}
\vspace{-1.5em}
\end{table}

%% file: tables/HallusionBench.tex
\begin{table*}[t]
\centering
\setlength{\tabcolsep}{3pt}
\begin{minipage}[t]{0.53\textwidth}
\centering
\caption{HallusionBench evaluation results.}
\label{tab:HallusionBench}
\vspace{1mm}
\renewcommand{\arraystretch}{1.2}
\resizebox{\linewidth}{!}{
\begin{tabular}{l*{5}{>{\centering\arraybackslash}p{1.45cm}}}
\toprule
Methods & fACC $\uparrow$ & qACC $\uparrow$ & ${easy}$A $\uparrow$ & ${hard}$A $\uparrow$ & aACC $\uparrow$ \\
\midrule
LLaVA-1.5\; & 17.9 \basex{0.0} & 8.13 \basex{0.0} & 36.0 \basex{0.0} & 36.7 \basex{0.0} & 41.5 \basex{0.0} \\
\ + OPERA & 16.2 \downbad{1.7} & 5.49 \downbad{2.6} & 37.6 \up{1.6} & 35.4 \downbad{1.3} & 41.2 \downbad{0.3} \\
\ + ICD & 13.9 \downbad{4.0} & 8.35 \up{0.2} & 36.9 \up{0.9} & 33.5 \downbad{3.2} & 38.2 \downbad{3.3} \\
\ + VCD & 13.9 \downbad{4.0} & 11.4 \up{3.3} & 33.0 \downbad{3.0} & 34.7 \downbad{2.0} & 41.1 \downbad{0.4} \\
\ + DeCo & 15.3 \downbad{2.6} & 8.13 \basex{0.0} & \sethlcolor{lightpink}\hl{42.1} \up{6.1} & 33.8 \downbad{2.9} & 41.3 \downbad{0.2} \\
\ + MemVR & \sethlcolor{lightpink}\hl{17.9} \basex{0.0} & 9.01 \up{0.9} & 36.9 \up{0.9} & 37.7 \up{1.0} & 42.5 \up{1.0} \\
\ \textbf{+ Ours} & 15.8 \downbad{2.1} & \sethlcolor{lightpink}\hl{13.8} \up{5.7} & 41.5 \up{5.5} & \sethlcolor{lightpink}\hl{43.4} \up{6.7} & \sethlcolor{lightpink}\hl{46.2} \up{4.7} \\
\midrule
Qwen-VL\; & 19.9 \basex{0.0} & 12.3 \basex{0.0} & 39.3 \basex{0.0} & 32.0 \basex{0.0} & 41.8 \basex{0.0} \\
+ DeCo\; & 20.2 \up{0.3} & \sethlcolor{lightpink}\hl{14.0} \up{1.7} & 42.8 \up{3.5} & 34.6 \up{2.6} & 43.6 \up{1.8} \\
\ \textbf{+ Ours} & \sethlcolor{lightpink}\hl{22.2} \up{2.3} & 13.8 \up{1.5} & \sethlcolor{lightpink}\hl{43.0} \up{3.7} & \sethlcolor{lightpink}\hl{35.5} \up{3.5} & \sethlcolor{lightpink}\hl{44.9} \up{3.1} \\
\bottomrule
\end{tabular}}
\end{minipage}
\hfill
\begin{minipage}[t]{0.45\textwidth}
\centering
\caption{LLaVA-Bench evaluation results.}
\label{tab:res_llavabench}
\vspace{0.5mm}
\renewcommand{\arraystretch}{1.0}
\resizebox{\linewidth}{!}{
\begin{tabular}{lcccc}
\toprule
\multirow{2}{*}[-0.5ex]{Method} & \multicolumn{4}{c}{LLaVA-Bench (in-the-wild)} \\
\cmidrule(lr){2-5}
 & Average$\uparrow$ & All\_1$\uparrow$ & All\_2$\uparrow$ & All\_3$\uparrow$ \\
\midrule
LLaVA-1.5 & 45.1 \basex{0.0} & 25.4 \basex{0.0} & 87.8 \basex{0.0} & 22.3 \basex{0.0} \\
{+ DeCo} & 45.6 \up{0.5} & 25.9 \up{0.5} & 87.9 \up{0.1} & 23.0 \up{0.7} \\
\textbf{+ Ours} & 45.9 \up{0.8} & 26.5 \up{1.1} & 88.0 \up{0.2} & 23.3 \up{1.0} \\
\midrule
Qwen-VL & 44.5 \basex{0.0} & 25.0 \basex{0.0} & 86.8 \basex{0.0} & 21.7 \basex{0.0} \\
{+ DeCo} & 44.7 \up{0.2} & 25.6 \up{0.6} & 85.8 \downbad{1.0} & 22.8 \up{1.1} \\
\textbf{+ Ours} & 44.9 \up{0.4} & 26.1 \up{1.1} & 86.3 \downbad{0.5} & 22.5 \up{0.8} \\
\midrule
GLM-4V & 49.1 \basex{0.0} & 33.2 \basex{0.0} & 85.8 \basex{0.0} & 28.5 \basex{0.0} \\
{+ DeCo} & 48.3 \downbad{0.8} & 31.8 \downbad{1.4} & 85.8 \basex{0.0} & 27.3 \downbad{1.2} \\
\textbf{+ Ours} & 49.3 \up{0.2} & 32.9 \downbad{0.3} & 86.5 \up{0.7} & 28.5 \basex{0.0} \\
\bottomrule
\end{tabular}}
\end{minipage}

\vspace{-1em}
\end{table*}

%% file: tables/Videohallucer.tex
\begin{table*}[t]
  \centering
  \caption{%
    VideoHallucer evaluation results of different methods. The best results are in \sethlcolor{lightpink}\hl{green}.}
  \label{tab:video-hallucination-main}
  \sethlcolor{lightpink}%
  \small
  \setlength{\tabcolsep}{5pt}
  \resizebox{0.9 \linewidth}{!}{%
  \begin{tabular}{@{}l cccccc@{}}
    \toprule
    \multicolumn{7}{c}{\textbf{Video-LLaVA-7B}} \\
    \midrule
    Method
    & Object--Relation$\uparrow$
    & Temporal$\uparrow$
    & Semantic$\uparrow$
    & Factual$\uparrow$
    & Non-Factual$\uparrow$
    & Overall$\uparrow$\\
    \midrule
    Regular
    & 0.355\,\basex{0.000} & 0.097\,\basex{0.000} & 0.120\,\basex{0.000} & 0.030\,\basex{0.000} & 0.260\,\basex{0.000}
    & 0.197\,\basex{0.000} \\
    \textsc{+ Self-Pep}
    & \hl{0.520}\,\up{0.165}
    & 0.063\,\downbad{0.034}
    & \hl{0.360}\,\up{0.240}
    & 0.110\,\up{0.080}
    & 0.340\,\up{0.080}
    & 0.432\,\up{0.235} \\
    + DeCo
    & 0.355\,\basex{0.000} & 0.097\,\basex{0.000} & 0.120\,\basex{0.000} & 0.030\,\basex{0.000} & 0.260\,\basex{0.000}
    & 0.197\,\basex{0.000} \\
    \textbf{+ Ours}
    & 0.410\,\up{0.055}
    & \hl{0.188}\,\up{0.091}
    & 0.270\,\up{0.150}
    & \hl{0.150}\,\up{0.120}
    & \hl{0.385}\,\up{0.125}
    & \hl{0.522}\,\up{0.325} \\
    \midrule
    \multicolumn{7}{c}{\textbf{LLaVA-NeXT-Video-7B-DPO}} \\
    \midrule
    Method
    & Object--Relation$\uparrow$
    & Temporal$\uparrow$
    & Semantic$\uparrow$
    & Factual$\uparrow$
    & Non-Factual$\uparrow$
    & Overall$\uparrow$\\
    \midrule
    Regular
    & 0.510\,\basex{0.000} & 0.278\,\basex{0.000} & 0.380\,\basex{0.000} & 0.145\,\basex{0.000} & 0.305\,\basex{0.000}
    & 0.609\,\basex{0.000} \\
    \textsc{+ Self-Pep}
    & 0.520\,\up{0.010}
    & 0.193\,\downbad{0.085}
    & 0.370\,\downbad{0.010}
    & 0.115\,\downbad{0.030}
    & 0.155\,\downbad{0.150}
    & \hl{0.744}\,\up{0.135} \\
    + DeCo
    & 0.510\,\basex{0.000} & 0.193\,\downbad{0.085} & 0.390\,\up{0.010} & 0.125\,\downbad{0.020} & 0.285\,\downbad{0.020}
    & 0.619\,\up{0.010} \\
    \textbf{+ Ours}
    & \hl{0.520}\,\up{0.010}
    & \hl{0.193}\,\downbad{0.085}
    & \hl{0.395}\,\up{0.015}
    & \hl{0.145}\,\basex{0.000}
    & \hl{0.300}\,\downbad{0.005}
    & 0.648\,\up{0.039} \\
    \bottomrule
  \end{tabular}
  }
\vspace{-1mm}
\end{table*}

%% file: tables/MME.tex
\begin{table}[t]
\centering
\caption{Performance evaluation on MME benchmark and MM-Vet. The best results are in \sethlcolor{lightpink}\hl{green}. }
\resizebox{1.0\linewidth}{!}{
\begin{tabular}{llcccccccc}
\toprule
\multirow{2}{*}[-0.5ex]{{Evaluation}} 
 & \multirow{2}{*}[-0.5ex]{{Methods}}  
 & \multicolumn{1}{c}{{MME-Hall}} 
 & \multicolumn{2}{c}{{Object-Level}}  
 & \multicolumn{2}{c}{{Attribute-Level}} 
 & \multicolumn{1}{c}{{Perception}} 
 & \multicolumn{1}{c}{{Cognition}} 
 & \multicolumn{1}{c}{{MM-Vet}} 
 \\ 
\cmidrule(lr){3-3} \cmidrule(lr){4-5} \cmidrule(lr){6-7} \cmidrule(lr){8-8} \cmidrule(lr){9-9} \cmidrule(l){10-10}
&  
& {Total} $\uparrow$
& {Existence} $\uparrow$ 
& {Count} $\uparrow$ 
& {Position} $\uparrow$
& {Color} $\uparrow$ 
& {Total} $\uparrow$
& {Total} $\uparrow$ 
& {Accuracy} $\uparrow$  
\\ 
\midrule

\multirow{6}{*}[-4ex]{{LLaVA-1.5}}
& {Regular} 
   & {643.3} \basex{0.0} 
   & {190.0} \basex{0.0} 
   & 155.0 \basex{0.0} 
   & {128.3} \basex{0.0} 
   & 170.0 \basex{0.0} 
   & 1508.97 \basex{0.0}
   & 335.71 \basex{0.0}
   & 27.50 \basex{0.0}
   \\

& OPERA  
   & {610.0} \downbad{33.} 
   & \sethlcolor{lightpink}\hl{195.0} \up{5.0} 
   & 128.3 \downbad{26.} 
   & {121.7} \downbad{6.6} 
   & 165.0 \downbad{5.0}
   & {1473.62} \downbad{35.}
   & {310.71} \downbad{45.}
   & {32.00} \up{4.5}
   \\

& VCD 
   & 648.3 \up{5.0} 
   & 190.0 \basex{0.0}
   & 155.0 \basex{0.0}
   & 133.3 \up{5.0}
   & 170.0 \basex{0.0}
   & 1515.01 \up{6.0}
   & 337.86 \up{2.1}
   & 26.70 \downbad{0.8}
   \\

& ICD  
   & 583.3 \downbad{60.}
   & 185.0 \downbad{5.0}
   & 130.0 \downbad{25.}
   & 121.7 \downbad{6.6}
   & 146.7 \downbad{23.}
   & 1306.91 \downbad{202}
   & 287.86 \downbad{67.}
   & 25.90 \downbad{1.6}
   \\

& DeCo  
   & 608.3 \downbad{35.}
   & 190.0 \basex{0.0}
   & 135.0 \downbad{20.}
   & 118.3 \downbad{10.}
   & 165.0 \downbad{5.0}
   & 1419.98 \downbad{89.}
   & 304.64 \downbad{31.}
   & 26.10 \downbad{1.4}
   \\

& {MemVR}
   & 648.3 \up{5.0}
   & 190.0 \basex{0.0}
   & 155.0 \basex{0.0}
   & 133.3 \up{5.0}
   & 170.0 \basex{0.0}
   & 1512.80 \up{9.7}
   & 335.71 \basex{0.0}
   & 32.40 \up{4.9}
   \\

& {AIR}
   & 638.3 \downbad{5.0}
   & 190.0 \basex{0.0}
   & 155.0 \basex{0.0}
   & 123.3 \downbad{5.0}
   & 170.0 \basex{0.0}
   & 1504.17 \downbad{4.8}
   & \sethlcolor{lightpink}\hl{372.50} \up{36.}
   & -
   \\

& {DMAS}
   & 659.9 \up{16.}
   & 195.0 \up{5.0}
   & 158.3 \up{3.3}
   & 133.3 \up{5.0}
   & 173.3 \basex{0.0}
   & - 
   & -
   & -
   \\

& \textbf{Ours}
   & \sethlcolor{lightpink}\hl{670.0} \up{26.}
   & 190.0 \basex{0.0}
   & \sethlcolor{lightpink}\hl{160.0} \up{5.0}
   & \sethlcolor{lightpink}\hl{145.0} \up{16.}
   & \sethlcolor{lightpink}\hl{175.0} \up{5.0}
   & \sethlcolor{lightpink}\hl{1525.92} \up{16.}
   & 341.78 \up{6.0}
   & \sethlcolor{lightpink}\hl{33.40} \up{5.9}
   \\ 
\midrule

\multirow{6}{*}[-4ex]{{Qwen-VL}}
& {Regular} 
   & 600.0 \basex{0.0}
   & 180.0 \basex{0.0}
   & 125.0 \basex{0.0}
   & 130.0 \basex{0.0}
   & 165.0 \basex{0.0}
   & {1442.79} \basex{0.0}
   & {342.14} \basex{0.0}
   & {31.10} \basex{0.0}
   \\

& OPERA  
   & \quad- \quad & \quad- \quad & \quad- \quad & \quad- \quad & \quad- \quad & \quad- \quad & \quad- \quad & \quad- \quad
   \\

& VCD 
   & 620.0 \up{20.}
   & {175.0} \downbad{5.0}
   & 125.0 \basex{0.0}
   & \sethlcolor{lightpink}\hl{145.0} \up{15.}
   & 175.0 \up{10.}
   & {1403.17} \downbad{39.}
   & 317.89 \downbad{24.}
   & {24.60} \downbad{6.5}
   \\
   
& ICD  
   & 591.6 \downbad{8.4}
   & 170.0 \downbad{10.}
   & 123.3 \downbad{1.7}
   & 138.3 \up{8.3}
   & 160.0 \downbad{5.0}
   & 1472.06 \up{29.}
   & 351.69 \up{9.5}
   & 21.70 \downbad{9.4}
   \\

& DeCo  
   & 613.3 \up{13.}
   & 180.0 \basex{0.0}
   & 125.0 \basex{0.0}
   & 133.3 \up{3.3}
   & \sethlcolor{lightpink}\hl{175.0} \up{10.}
   & 1489.53 \up{46.}
   & 356.21 \up{14.}
   & 31.50 \up{0.4}
   \\

& {MemVR}
   & {610.0} \up{10.}
   & 180.0 \basex{0.0}
   & 130.0 \up{5.0}
   & 130.0 \basex{0.0}
   & 170.0 \up{5.0}
   & 1473.45 \up{30.}
   & 347.50 \up{5.4}
   & \sethlcolor{lightpink}\hl{33.40} \up{2.3}
   \\

& {AIR}
   & \sethlcolor{lightpink}\hl{636.7} \up{36.}
   & \sethlcolor{lightpink}\hl{185.0} \basex{0.0}
   & 145.0 \up{5.0}
   & 126.7 \basex{0.0}
   & 180.0 \up{5.0}
   & 1476.87 \up{34.}
   & 352.14 \up{10.}
   & -
   \\

& {DMAS}
   & {633.3} \up{33.}
   & 170.0 \basex{0.0}
   & \sethlcolor{lightpink}\hl{145.0} \up{5.0}
   & 133.3 \basex{0.0}
   & \sethlcolor{lightpink}\hl{185.5} \up{5.0}
   & -
   & -
   & -
   \\

& \textbf{Ours}
   & 621.6 \up{21.}
   & 180.0 \basex{0.0}
   & 135.0 \up{10.}
   & 141.6 \up{11.}
   & 165.0 \basex{0.0}
   & \sethlcolor{lightpink}\hl{1513.79} \up{71.}
   & \sethlcolor{lightpink}\hl{357.14} \up{15.}
   & {32.80} \up{1.7}
   \\ 
\midrule

\multirow{5}{*}[-4ex]{{GLM-4V}}
& {Regular} 
   & 703.3 \basex{0.0}
   & 200.0 \basex{0.0}
   & 168.3 \basex{0.0}
   & 156.7 \basex{0.0}
   & 178.3 \basex{0.0}
   & 1680.89 \basex{0.0}
   & 479.64 \basex{0.0}
   & 37.60 \basex{0.0}
   \\

& OPERA  
   & - & - & - & - & -
   & - & - & -
   \\

& VCD 
   & 675.0 \downbad{28.}
   & 200.0 \basex{0.0}
   & 166.7 \downbad{1.6}
   & 138.3 \downbad{18.}
   & 170.0 \downbad{8.3}
   & 1624.10 \downbad{56.}
   & 481.43 \up{1.8}
   & 33.60 \downbad{4.0}
   \\

& ICD  
   & 628.3 \downbad{75.}
   & 200.0 \basex{0.0}
   & 151.7 \downbad{16.}
   & 113.3 \downbad{43.}
   & 163.3 \downbad{15.}
   & 1566.09 \downbad{114}
   & \sethlcolor{lightpink}\hl{508.21} \up{28.}
   & 32.50 \downbad{5.1}
   \\

& DeCo  
   & 610.5 \downbad{92.}
   & 200.0 \basex{0.0}
   & 160.0 \downbad{8.3}
   & 105.5 \downbad{51.}
   & 145.0 \downbad{33.}
   & 1505.86 \downbad{175}
   & 443.57 \downbad{36.}
   & 34.40 \downbad{3.2}
   \\

& {MemVR}
   & 703.3 \basex{0.0}
   & 200.0 \basex{0.0}
   & 173.3 \up{5.0}
   & 151.7 \downbad{5.0}
   & 178.3 \basex{0.0}
   & 1683.03 \up{2.1}
   & 487.14 \up{7.5}
   & 38.20 \up{0.6}
   \\

& {AIR}
   & 703.3 \basex{0.0}
   & 200.0 \basex{0.0}
   & 173.3 \up{5.0}
   & 151.7 \downbad{5.0}
   & 178.3 \basex{0.0}
   & 1681.89 \up{1.0}
   & 479.64 \basex{0.0}
   & -
   \\

& \textbf{Ours}
   & \sethlcolor{lightpink}\hl{709.9} \up{6.6}
   & \sethlcolor{lightpink}\hl{200.0} \basex{0.0}
   & \sethlcolor{lightpink}\hl{178.3} \up{10.}
   & \sethlcolor{lightpink}\hl{153.3} \downbad{3.4}
   & \sethlcolor{lightpink}\hl{178.3} \basex{0.0}
   & \sethlcolor{lightpink}\hl{1692.02} \up{11.}
   & 496.78 \up{17.}
   & \sethlcolor{lightpink}\hl{38.80} \up{1.2}
   \\
\bottomrule
\end{tabular}
}\vspace{-1em}
\label{tab:MME}
\end{table}

%% file: tables/ablation.tex




\begin{wraptable}{R}{0.5\textwidth}
\centering
\vspace{-1em}
\caption{Ablation results on CHAIR benchmark. The best results are in \sethlcolor{lightpink}\hl{green}.}
\resizebox{1.0 \linewidth}{!}{
\begin{tabular}{lccccc}
\toprule
 Methods  & CHAIR$_S$ $\downarrow$ & CHAIR$_I$ $\downarrow$ & {Average} $\downarrow$ & {Len} & {Recall} $\uparrow$ \\ 
\midrule
\ LLaVA-1.5  & 50.0 \basexdown{0.0} & 15.4 \basexdown{0.0} & 32.7 \basexdown{0.0} & 100.6 & 77.1 \basex{0.0} \\
\ + LaSCD-$\alpha$ & 47.0 \down{3.0} & 12.3 \down{3.1} & 29.7 \down{3.0} & 99.4 & 78.5 \up{1.4}  \\
\ + LaSCD-$\beta$ & 43.4 \down{6.6} & 12.2 \down{3.2} & 27.8 \down{4.9} & 98.9 & 75.3 \downbad{1.8}  \\
\ + $\mathcal{K}_\nabla$ & 43.4 \down{6.6} & \sethlcolor{lightpink}\hl{11.7} \down{3.7} & 27.6 \down{5.1} & 98.9 & 76.8 \downbad{0.3}  \\
\ + $\mathcal{K}_{1D}$ & 40.8 \down{9.2} & 12.0 \down{3.4} & 26.4 \down{6.3} & 100.0 & 77.9 \up{0.8}  \\ 
\ + $E_{\mathrm{txt}}$ & 48.6 \down{1.4} & 14.2 \down{1.2} & 31.4 \down{1.3} & 101.0 & 81.5 \up{4.4} \\
\ + $m^{(l)}$ & 48.0 \down{2.0} & 12.9 \down{2.5} & 30.4 \down{2.3} & 98.7 & 78.3 \up{1.2} \\
\ + $\mathrm{H}_{\text{Shannon}}$ & 45.6 \down{4.4} & 12.1 \down{3.3} & 28.8 \down{2.3} & 98.8 & 78.2 \up{1.1} \\
\ + $\mathcal{L} \in \{20\}$ & 45.4 \down{4.6} & 12.5 \down{2.9} & 29.0 \down{3.7} & 99.1 & 78.1 \up{1.0} \\
\ + $\mathcal{L} \in \{1,\ldots,L\}$ & 44.0 \down{6.0} & 12.7 \down{2.7} & 28.4 \down{4.3} & 98.3 & 77.3 \up{0.2}  \\ 
\midrule
\ \textbf{+ Ours} & \sethlcolor{lightpink}\hl{40.6} \down{9.4} & 11.8 \down{3.6} & \sethlcolor{lightpink}\hl{26.2} \down{6.5} & 100.0 & 77.6 \up{0.5}  \\ 
\bottomrule
\end{tabular}
}
\label{tab:CHAIR_ab}
\vspace{-1em}
\end{wraptable}

%% file: sections/conclusion.tex
\section{Conclusion}
\label{sec:conclusion}

In this paper, we show that visual hallucination in MLLMs is not reliably explained by insufficient visual-attention magnitude, while Laplacian energy exposes both hallucinated commitment and transient ground-truth recovery. Motivated by this finding, we propose \sexyname, a training-free decoding strategy that selects informative layers through high-frequency energy and remaps logits in closed form. Experiments across several benchmarks show that this spectral attention signal yields more faithful outputs without sacrificing general capabilities. We hope this finding offers a useful starting point for future studies of when and where hallucinations emerge in MLLMs.

%% file: sections/appendix.tex
\appendix

\section{Theoretical Analysis}
\label{apx:theory}

This section gives a lightweight geometric justification for using Laplacian energy as a
layer-wise hallucination signal. The goal is not to fully model the internal dynamics of an
MLLM, but to isolate one structural property that our empirical analysis repeatedly observes:
\emph{faithful visual grounding tends to be spatially coherent, whereas hallucination-prone
states often exhibit fragmented or high-curvature visual attention}. We therefore analyze a
normalized visual-attention distribution, which fixes the total visual mass and exposes the
effect of spatial structure alone.

\paragraph{Setup: normalized grounding geometry.}
Let visual tokens form a 2-D grid graph $\mathcal{G}=(\mathcal{V}_g,\mathcal{E})$ with $N_v$
nodes, where edges connect spatial neighbors. At layer $\ell$, let
\begin{equation}
    \boldsymbol{\alpha}^{(\ell)}\in\mathbb{R}_{\ge 0}^{N_v},
    \quad
    \mathbf{1}^{\top}\boldsymbol{\alpha}^{(\ell)}=1
\end{equation}
be the visual-attention slice after normalization over visual tokens. For a grounded
question, assume the correct answer $y^*$ depends primarily on a coherent region
$R^*\subseteq\mathcal{V}_g$ (e.g., an object, an attribute-bearing part, or a pair of
spatially related regions). A faithful grounding query should concentrate attention on
$R^*$ and vary smoothly inside it, since neighboring patches on the same visual entity
provide correlated evidence.

Let $\Delta$ be the discrete grid Laplacian and define the normalized high-frequency energy
\begin{equation}
    E(\boldsymbol{\alpha}) = \|\Delta\boldsymbol{\alpha}\|_2.
    \label{eq:theory-energy}
\end{equation}
In 1-D, $[\Delta\boldsymbol{\alpha}]_i=\alpha_{i+1}-2\alpha_i+\alpha_{i-1}$; in 2-D this
is the standard five-point Laplacian. $E(\boldsymbol{\alpha})$ is small when attention is
smooth over the grid, and large when neighboring tokens receive sharply different weights.

\paragraph{Proposition 1 (coherent grounding has bounded high-frequency energy).}
Consider the ideal region-uniform attention
\begin{equation}
    \boldsymbol{\alpha}_{R^*}(i)
    =
    \begin{cases}|R^*|^{-1}, & i\in R^*,\\ 0, & i\notin R^*.\end{cases}
    \label{eq:region-attn}
\end{equation}
Let $\partial R^*$ denote the grid boundary of $R^*$. Then
\begin{equation}
    E(\boldsymbol{\alpha}_{R^*})^2
    \;\le\;
    \frac{C_{\mathcal{G}}\,|\partial R^*|}{|R^*|^2},
    \label{eq:coherent-bound}
\end{equation}
where $C_{\mathcal{G}}$ depends only on the local grid degree.

\begin{proof}
For nodes strictly inside $R^*$, all neighbors share the same value $|R^*|^{-1}$, so
$[\Delta\boldsymbol{\alpha}_{R^*}]_i=0$. For nodes strictly outside $R^*$ and not adjacent
to its boundary, all values are zero, so the Laplacian is again zero. Nonzero responses
occur only at nodes adjacent to $\partial R^*$; each such response has magnitude at most
$C_{\mathcal{G}}|R^*|^{-1}$ for a degree-dependent constant $C_{\mathcal{G}}$. Summing
squared responses over the $O(|\partial R^*|)$ boundary-adjacent nodes yields
Eq.~\eqref{eq:coherent-bound}.
\end{proof}

\noindent
The bound shows that coherent region grounding is spectrally low-frequency: high-frequency
energy accumulates only around object boundaries, not throughout the answer-relevant
region. Larger or more compact regions, which have a smaller
$|\partial R^*|/|R^*|^2$ ratio, therefore incur lower normalized Laplacian energy.

\paragraph{Proposition 2 (fragmentation increases squared high-frequency energy).}
Suppose the same total support area $A=|R^*|$ is split into $K\ge 2$ disconnected
components $R_1,\ldots,R_K$, each with area $\approx A/K$, and define
\begin{equation}
    \boldsymbol{\alpha}_{\mathrm{frag}}(i)
    =
    \begin{cases}A^{-1}, & i\in\bigcup_{k=1}^K R_k,\\ 0, & \text{otherwise.}\end{cases}
\end{equation}
Then
\begin{equation}
    E(\boldsymbol{\alpha}_{\mathrm{frag}})^2
    \;\asymp\;
    \frac{\sum_{k=1}^{K}|\partial R_k|}{A^2}.
    \label{eq:frag-energy}
\end{equation}
On a 2-D grid, the discrete isoperimetric inequality implies that a compact component of
area $A/K$ has boundary $\Omega(\sqrt{A/K})$, so
\begin{equation}
    \sum_{k=1}^{K}|\partial R_k|
    \;=\;
    \Omega\!\left(K\sqrt{A/K}\right)
    =
    \Omega\!\left(\sqrt{KA}\right).
    \label{eq:isoperimetric-frag}
\end{equation}
A single compact region of area $A$ has boundary $O(\sqrt{A})$, so splitting into $K$
pieces inflates the boundary term and hence $E(\boldsymbol{\alpha}_{\mathrm{frag}})^2$
by a factor $\Omega(\sqrt{K})$ at the level of this boundary-scaling approximation.

\noindent
Propositions~1--2 establish the geometric core of the argument: coherent grounding has
bounded spatial curvature, whereas fragmented attention incurs additional spectral cost
through its enlarged boundary.

\paragraph{Theorem 1 (incoherent perturbations raise expected squared energy).}
When the model fails to localize $R^*$, its visual attention can degrade into a
\emph{spatially incoherent search}: attention mass is spread across tokens that are
irrelevant or geometrically dispersed, producing a distribution with high spatial variance.
As a local model on the probability simplex, write
\begin{equation}
    \boldsymbol{\alpha}_{\mathrm{hall}}
    =
    \boldsymbol{\pi}+\boldsymbol{\eta},
    \quad
    \mathbf{1}^{\top}\boldsymbol{\pi}=1,
    \quad
    \mathbf{1}^{\top}\boldsymbol{\eta}=0,
    \quad
    \mathbb{E}[\boldsymbol{\eta}]=\mathbf{0},
    \quad
    \mathrm{Cov}(\boldsymbol{\eta})=\Sigma_{\eta}\succeq 0,
    \label{eq:hall-model}
\end{equation}
where $\boldsymbol{\pi}$ is an uninformative baseline (e.g., near-uniform) and
$\boldsymbol{\eta}$ is a zero-mean perturbation in the simplex tangent space. Equivalently,
$\Sigma_\eta\mathbf{1}=\mathbf{0}$. We assume the perturbation is small enough that
$\boldsymbol{\pi}+\boldsymbol{\eta}$ remains nonnegative, or view Eq.~\eqref{eq:hall-model}
as a local approximation around $\boldsymbol{\pi}$. The only structural condition is that
$\Sigma_\eta$ has non-negligible projection onto the high-pass subspace of $\Delta$:
\begin{equation}
    \sigma_\Delta^2
    \;\triangleq\;
    \frac{\operatorname{Tr}(\Delta^\top\Delta\,\Sigma_\eta)}
         {\operatorname{Tr}(\Delta^\top\Delta)}
    \;>\;0.
    \label{eq:hp-var}
\end{equation}

\begin{theorem}
\label{thm:energy-amplification}
Under the grounding-failure model~\eqref{eq:hall-model} with $\sigma_\Delta^2>0$,
\begin{equation}
    \mathbb{E}\!\left[E(\boldsymbol{\alpha}_{\mathrm{hall}})^2\right]
    =
    \|\Delta\boldsymbol{\pi}\|_2^2
    +
    \underbrace{\operatorname{Tr}(\Delta^\top\Delta\,\Sigma_\eta)}_{=\,\sigma_\Delta^2\,\operatorname{Tr}(\Delta^\top\Delta)\,>\,0}
    \;>\;
    \|\Delta\boldsymbol{\pi}\|_2^2
    \;=\;
    \mathbb{E}\!\left[E(\boldsymbol{\alpha}_{\mathrm{hall}})^2\right]\big|_{\Sigma_\eta=0}.
    \label{eq:energy-amplification}
\end{equation}
Moreover, because the norm is convex,
\begin{equation}
    \mathbb{E}\!\left[E(\boldsymbol{\alpha}_{\mathrm{hall}})\right]
    \;\ge\;
    E(\boldsymbol{\pi}).
\end{equation}
\end{theorem}

\begin{proof}
Expand
\begin{equation}
    \mathbb{E}\|\Delta(\boldsymbol{\pi}+\boldsymbol{\eta})\|_2^2
    =
    \|\Delta\boldsymbol{\pi}\|_2^2
    + 2(\Delta\boldsymbol{\pi})^\top\mathbb{E}[\Delta\boldsymbol{\eta}]
    + \mathbb{E}\|\Delta\boldsymbol{\eta}\|_2^2 .
\end{equation}
The cross term vanishes because $\mathbb{E}[\boldsymbol{\eta}]=\mathbf{0}$. The residual
term equals
\begin{equation}
    \mathbb{E}\|\Delta\boldsymbol{\eta}\|_2^2
    =
    \operatorname{Tr}(\Delta^\top\Delta\,\Sigma_\eta)
    =
    \sigma_\Delta^2\operatorname{Tr}(\Delta^\top\Delta)
    >0
\end{equation}
by condition~\eqref{eq:hp-var}. For the unsquared energy, Jensen's inequality applied to
the convex function $\boldsymbol{x}\mapsto\|\Delta\boldsymbol{x}\|_2$ gives
\begin{equation}
    \mathbb{E}\|\Delta(\boldsymbol{\pi}+\boldsymbol{\eta})\|_2
    \ge
    \|\Delta(\boldsymbol{\pi}+\mathbb{E}\boldsymbol{\eta})\|_2
    =
    \|\Delta\boldsymbol{\pi}\|_2.
\end{equation}
\end{proof}

\paragraph{Implication.}
The normalized analysis supports a modest but useful conclusion: when total visual mass is held fixed, coherent region localization has bounded Laplacian energy, while fragmentation and high-pass stochastic variation increase the spectral cost. The actual \sexyname score uses the unnormalized visual slice, combineing this spatial-shape term with the visual-attention mass as shown in Eq.~\eqref{eq:energy}. We use the theory as a
geometric explanation for the spatial component of the score, and rely on the empirical results in Sec.~\ref{sec:finding1}--\ref{sec:finding2} to justify the mass-weighted version used for decoding.

\section{Experimental Details}
\label{apx:experimental_details}

\subsection{Benchmarks and Metrics}
\label{apx:benchmarks}

\textbf{Hallucination benchmarks.}
We evaluate image-level hallucination on POPE~\cite{li2023evaluating}, CHAIR~\cite{rohrbach2018object}, and HallusionBench~\cite{guan2024hallusionbench}, and further evaluate video hallucination on VideoHallucer~\cite{videohallucer}. POPE is a binary VQA-style benchmark that probes object existence with questions of the form ``Is [object] in this image?''; we append the instruction ``Please answer yes or no.'' and report performance under the random, popular, and adversarial splits. CHAIR evaluates object hallucination in generated captions. Following standard practice, we use images from the MSCOCO Val2014 split~\cite{lin2014microsoft} and the prompt ``Please describe this image in detail.'' We report CHAIR$_S$, which measures the fraction of captions containing at least one hallucinated object, and CHAIR$_I$, which measures the fraction of mentioned objects that are hallucinated:
\begin{equation}
\mathrm{CHAIR}_I =
\frac{|\{\text{hallucinated objects}\}|}
{|\{\text{all mentioned objects}\}|},
\qquad
\mathrm{CHAIR}_S =
\frac{|\{\text{sentences with hallucinated objects}\}|}
{|\{\text{all sentences}\}|}.
\end{equation}
HallusionBench evaluates multimodal hallucination with a broader set of visual reasoning and language biases; we follow the official evaluation protocol. VideoHallucer evaluates hallucination in large video-language models with adversarial binary VideoQA pairs, covering both intrinsic hallucinations, such as object-relation and temporal hallucinations, and extrinsic hallucinations, including factual and non-factual hallucinations.

\textbf{General multimodal benchmarks.}
To ensure that hallucination mitigation does not degrade general ability, we evaluate on MME~\cite{fu2023mme}, MM-Vet~\cite{2024MMVet}, and LLaVA-Bench-in-the-wild~\cite{liu2024visual}. MME contains perception and cognition categories and uses yes/no questions. MM-Vet evaluates integrated multimodal capabilities with open-ended responses. For LLaVA-Bench-in-the-wild, we follow the GPT-based evaluation protocol used by the benchmark. All benchmarks that require external model-based judging, including MM-Vet, HallusionBench and LLaVA-Bench-in-the-wild, are evaluated with the GPT-4 API.

\subsection{Inference Protocol}
\label{apx:inference_protocol}

We use greedy decoding as the default decoding strategy for \sexyname{} and the original MLLM unless a baseline requires sampling or beam search by design. Questions from each benchmark are formatted using the chat template of the evaluated MLLM. For binary VQA-style benchmarks such as POPE, MME, and VideoHallucer, we parse the model output into yes/no answers; for caption-style evaluation such as CHAIR, we use the standard object-matching protocol. For GPT-based evaluations, we use the official or benchmark-recommended evaluator when available and run all external model-based judging with the GPT-4 API.

\textbf{Baseline settings.}
We compare against training-free hallucination mitigation methods that intervene at decoding time. Our baseline reproduction follows MemVR~\cite{zoulook} in evaluation protocol and practice: whenever a method releases official code, we re-run it with its released scripts and paper-aligned hyper-parameters; when no implementation is available, we copy the numbers directly from the original paper's tables. For any model--benchmark pair that a baseline paper does not report, we leave the corresponding table entry blank rather than extrapolating. For DeCo~\cite{wang2025mllm} specifically, we \emph{manually reproduce} every dataset--model combination absent from its published evaluations.

\textbf{\sexyname{} settings.}
For \sexyname{} and the original MLLM, we use greedy decoding by default and keep the same prompt template, image/video preprocessing, maximum generation length, and answer parsing within each benchmark. During \sexyname{} decoding, the model is run once per generation step with attentions and hidden states returned, so that we can read both the final prediction and the intermediate-layer predictions without a second forward pass.

The algorithm has three parameterized components. First, for \textbf{layer selection}, we compute high-frequency energy on the visual-attention slice of candidate decoder layers. In the main setting, we use the vision-token energy mode and the standard 2-D discrete Laplacian operator with kernel $\left[\begin{smallmatrix}0&1&0\\1&-4&1\\0&1&0\end{smallmatrix}\right]$. When the visual tokens cannot be reshaped into a square grid, we use a 1-D second-order difference kernel $[1,-2,1]$ as the fallback. For all the MLLMs, the candidate layer range is set within $[8,29)$. Sobel, Laplacian-of-Gaussian, and text-token energy variants are used only for ablation.

Second, for \textbf{candidate-token filtering}, we restrict logit remapping to a plausible token set selected from the final-layer distribution. Unless otherwise stated, we use top-$k{=}10$ and cumulative probability threshold $p{=}0.9$. These thresholds prevent the contrastive term from promoting tokens that the final layer already assigns very low probability.

Third, for \textbf{logit remapping}, $\alpha$ controls the strength of suppressing the high-energy layer's preference. We set $\alpha=0.1$ in the main experiments. We set $\beta=0.2$ for HallusionBench and 0.6 for CHAIR evaluation. All parameters are fixed for full evaluation.

\subsection{Backbone Adaptation and Compatibility}
\label{apx:backbone_compatibility}

\sexyname{} requires two kinds of intermediate signals from the backbone: the last-query attention over visual tokens and the LM-head logits from selected intermediate layers. Different MLLM implementations expose these signals differently, so we use lightweight model-specific adapters while keeping the original image/video encoder, prompt template, and tokenizer unchanged.

\textbf{LLaVA-Series.}
For LLaVA-1.5, LLaVA-Next-Video-DPO and Video-LLaVA, we insert the same image/video placeholder tokens used by the original model and force the attention implementation to an eager mode so that attention maps can be captured during generation. The adapter records the expanded multimodal sequence length after visual features are inserted and uses the placeholder position to align the visual-token span in the KV sequence. For video models, the visual tokens are produced from frame-level features and remain contiguous in the multimodal sequence; Video-LLaVA explicitly uses one image placeholder per sampled frame, while LLaVA-Next-Video uses the model's multimodal preparation routine to expand the video input. Intermediate-layer logits are obtained from the same forward pass through a forward-capture hook, avoiding an additional decoding pass.

\textbf{GLM-4V.}
GLM-4V does not expose attention in the same format as LLaVA-style models. We therefore keep its original core-attention forward path unchanged and recompute the attention probabilities in a side channel for the monitored layers. This design avoids changing the hidden-state trajectory when the intervention strength is zero. The visual-token span is located using the begin-of-image marker when available; otherwise, we infer the number of inserted visual tokens from the KV length after multimodal expansion and use a conservative fallback start position. Intermediate logits are obtained by applying the model's output layer to hooked block hidden states, with the final layer normalization used when required by the backbone.

\textbf{Qwen-VL.}
For Qwen-VL, we use the model's returned self-attention tensors and register hooks on the selected transformer blocks to obtain intermediate hidden states. The visual-token span is identified from the tokenizer's image start/end markers, matching the model's convention that image embeddings replace the tokens between these markers. Intermediate logits are computed by applying the final normalization and output embedding head to the hooked hidden states.

\textbf{Grid compatibility of the Laplacian energy.}
Our default high-frequency measure uses the Frobenius norm of the 2-D discrete Laplacian response on a reshaped visual-attention map. This operation is meaningful only when the selected visual tokens preserve a spatial grid order. LLaVA-style image encoders and video encoders generally emit patch features in raster order, and spatial pooling in LLaVA-Next keeps a regular local neighborhood structure for each frame. However, any-resolution padding/unpadding, frame concatenation, and model-specific visual resampling can make the full visual-token sequence non-square or not strictly aligned with a single 2-D semantic layout.

\subsection{Reproducibility Details}
\label{apx:reproducibility}

We fix random seeds for sampling, subset selection, and repeated evaluation when applicable. For CHAIR, we evaluate on same COCO Val2014 subsets. For POPE, we evaluate all three sampling protocols (random, popular, and adversarial). For MME and VideoHallucer, we use yes/no answer parsing and skip non-binary samples only in analyses where binary correctness is required. All compared methods use the same image or video inputs, prompts, and answer parsing rules under a given benchmark.

For analysis figures that inspect layer-wise behavior, we run the model with both hidden states and attentions returned. We compute per-layer logits by applying the LM head to normalized intermediate hidden states and compute high-frequency energy on the last-query visual-attention slice. These analysis runs are used only to diagnose layer behavior and produce figures; main benchmark evaluation uses the decoding protocol described above.

\subsection{Computational Resources}
\label{apx:compute}

All experiments are implemented in PyTorch~\cite{pytorch} and conducted on NVIDIA A800 80GB GPUs.

\subsection{Artifacts and Licenses}
\label{apx:artifacts}

We follow the licenses of all datasets, model checkpoints, and evaluation tools used in the experiments, and restrict their use to academic research when required. The evaluated datasets and benchmarks include MSCOCO, POPE, CHAIR, HallusionBench, VideoHallucer, MME, MM-Vet, and LLaVA-Bench-in-the-wild. The evaluated model families include LLaVA, Qwen-VL, GLM-4V, InstructBLIP, Video-LLaVA, and LLaVA-Next-Video, subject to the licenses of their released checkpoints and code. We will release code and instructions for reproducing the reported evaluations and analysis figures.

\section{Limitations}
\label{apx:limitations}

\sexyname{} relies on intermediate attention signals, which creates several practical limitations. First, the method assumes that the selected visual-token span preserves a meaningful local structure. In our main implementation, high-frequency energy is computed with a 2-D discrete Laplacian when the visual tokens can be reshaped into a square grid, and with a 1-D second-order difference otherwise. This design works naturally for patch-based vision encoders whose token order follows the image grid, but the interpretation can become less direct for models that use aggressive token pooling, learned resamplers, token pruning, or video-specific mixing across frames. In these cases, the energy still measures local variation along the model-provided token order, but it may not perfectly correspond to semantic spatial inconsistency in the original visual input. Second, the backbone must provide accessible self-attention maps and intermediate-layer logits; models or serving systems that hide these internal states cannot directly support our current implementation. Third, our evaluation focuses on hallucination and general multimodal benchmarks, and we have not systematically tested whether the same intervention improves or affects reasoning-intensive tasks. Finally, although we cover several representative image and video MLLMs, we have not evaluated substantially larger-scale models, so the behavior of \sexyname{} under stronger scaling regimes remains an open question.

\begin{table*}[t]
\centering
\caption{Comparison of decoding-time hallucination mitigation methods.}
\label{tab:method_positioning}
\resizebox{\linewidth}{!}{
\begin{tabular}{lcccc>{\centering\arraybackslash}p{2.6cm}}
\toprule
Methods
& \begin{tabular}[c]{@{}c@{}}No Second\\Input View\end{tabular}
& \begin{tabular}[c]{@{}c@{}}Hallucination\\Signal\end{tabular}
& \begin{tabular}[c]{@{}c@{}}Visual-Structure\\Aware\end{tabular}
& \begin{tabular}[c]{@{}c@{}}Modified\\Component(s)\end{tabular}
& \begin{tabular}[c]{@{}c@{}}Decoding\\Paradigm\end{tabular} \\
\midrule
DoLa~\cite{chuang2023dola}
& \cmark & \cmark & \xmark & Logits & Contrast \\
OPERA~\cite{opera2024cvpr}
& \cmark & \xmark & \xmark & Attention matrix & Att-intervention \\
PAI~\cite{liu2024paying}
& \xmark & \xmark & \xmark & Attention matrix & Att-intervention \\
VCD~\cite{vcd2024cvpr}
& \xmark & \xmark & \xmark & Visual input, logits & Contrast \\
ICD~\cite{wang2024mitigating}
& \xmark & \xmark & \xmark & Textual input, logits & Contrast \\
SID~\cite{huo2024self}
& \xmark & \xmark & \xmark & Textual input, logits & Contrast \\
DeCo~\cite{wang2025mllm}
& \cmark & \xmark & \xmark & Logits & Contrast \\
MemVR~\cite{zoulook}
& \cmark & \cmark & \xmark & Hidden states & Visual-retracing \\
AIR~\cite{zhu2026look}
& \cmark & \cmark & \xmark & Hidden states & Visual-retracing \\
DMAS~\cite{yin2026dynamic}
& \cmark & \xmark & \xmark & Attention matrix & Att-intervention \\
\cellcolor{lightgray}\textbf{\sexyname{} (ours)}
& \cellcolor{lightgray}\cmark
& \cellcolor{lightgray}\cmark
& \cellcolor{lightgray}\cmark
& \cellcolor{lightgray}Logits
& \cellcolor{lightgray}\textbf{Spectral-contrast} \\
\bottomrule
\end{tabular}
}
\end{table*}

\section{Case Study}
\label{apx:case_study}
In this section, we present more cases across to support the empirical findings in Sec.~\ref{sec:finding1} and Sec.~\ref{sec:finding2}.

\begin{figure}[ht]
    \centering
    \includegraphics[width=1.0\linewidth]{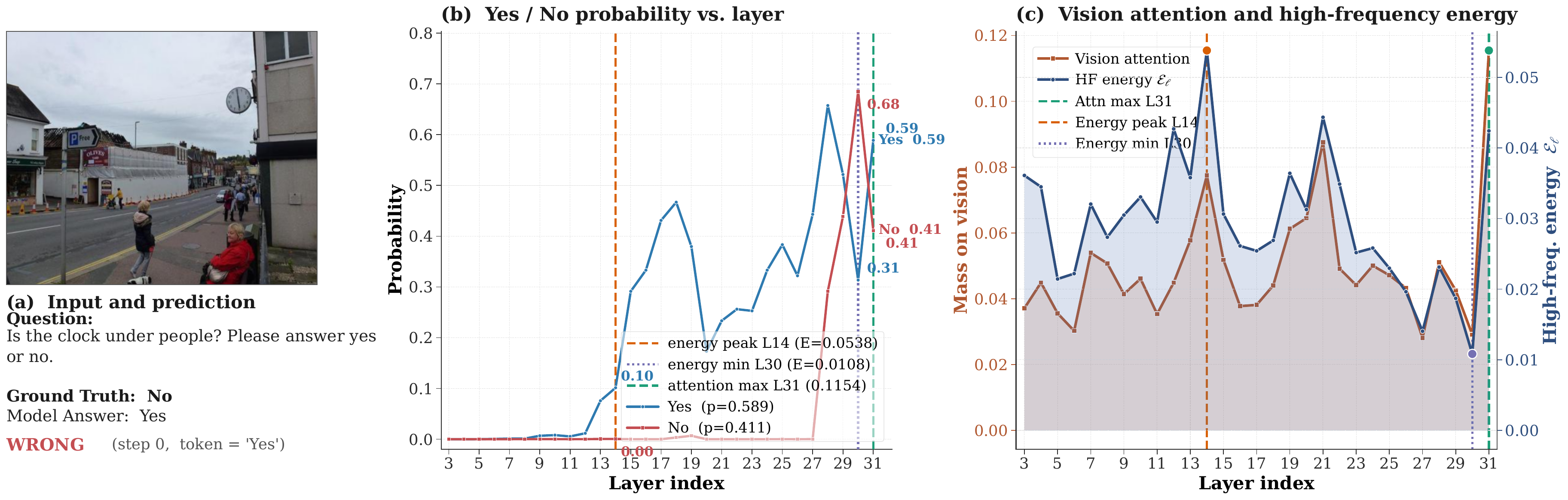}
    \caption{A case study of hallucination during decoding progress.}
    \label{fig:vis1}
\end{figure}

\begin{figure}[ht]
    \centering
    \includegraphics[width=1.0\linewidth]{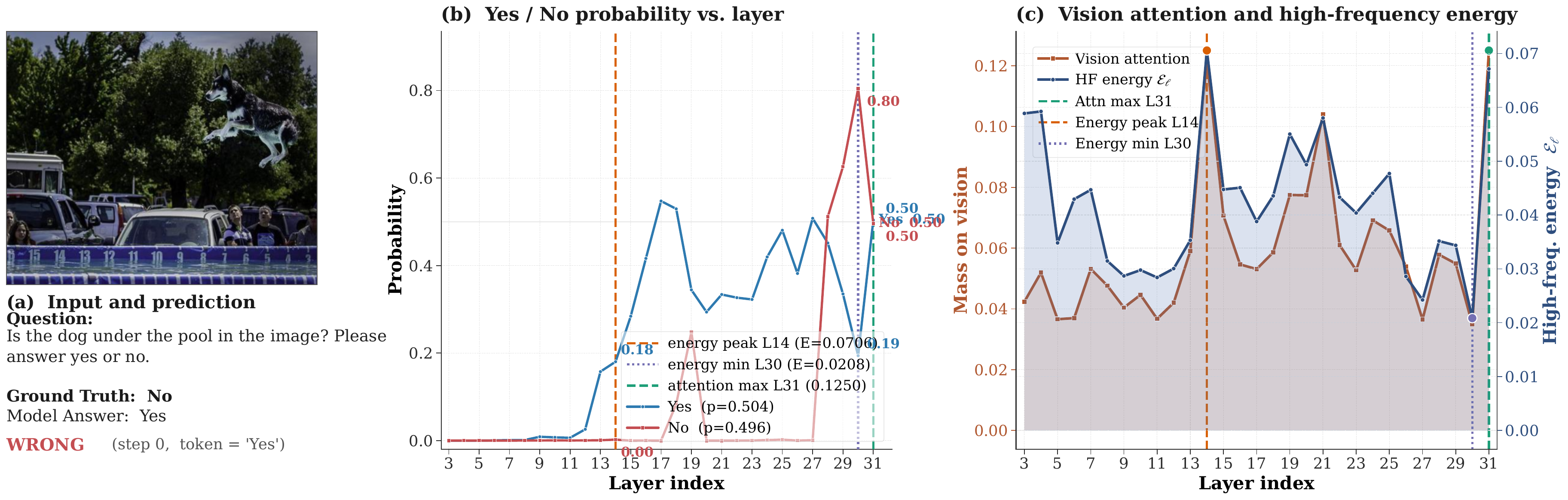}
    \caption{A case study of hallucination during decoding progress.}
    \label{fig:vis2}
\end{figure}

\begin{figure}[ht]
    \centering
    \includegraphics[width=1.0\linewidth]{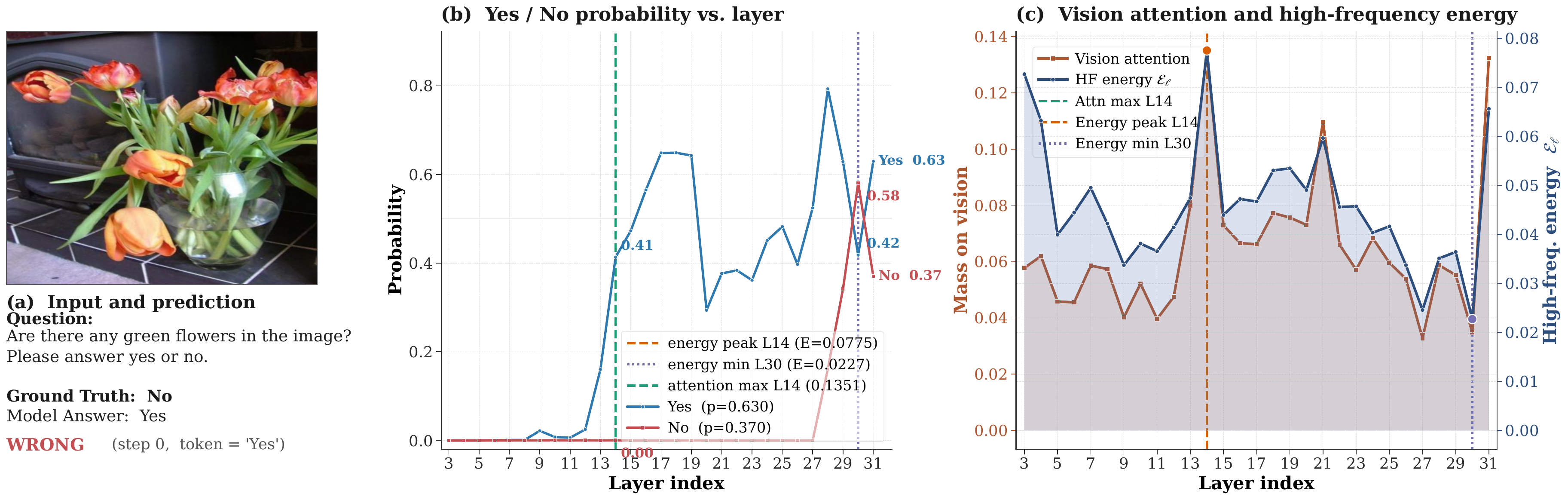}
    \caption{A case study of hallucination during decoding progress.}
    \label{fig:vis3}
\end{figure}

\begin{figure}[ht]
    \centering
    \includegraphics[width=1.0\linewidth]{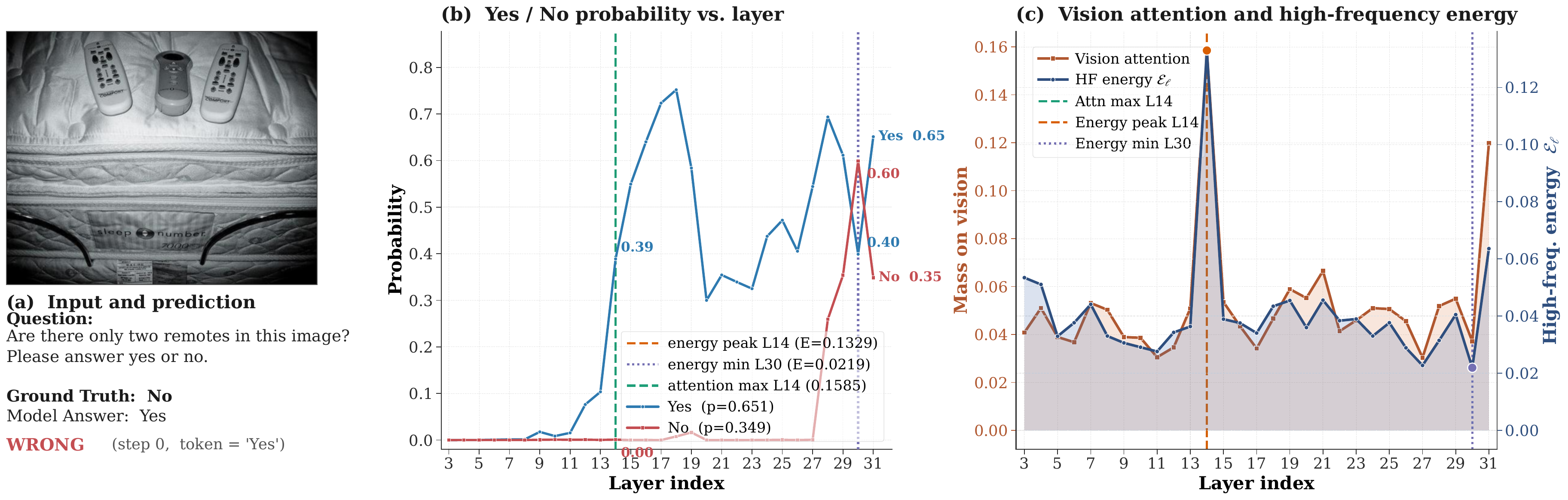}
    \caption{A case study of hallucination during decoding progress.}
    \label{fig:vis4}
\end{figure}

\begin{figure}[ht]
    \centering
    \includegraphics[width=1.0\linewidth]{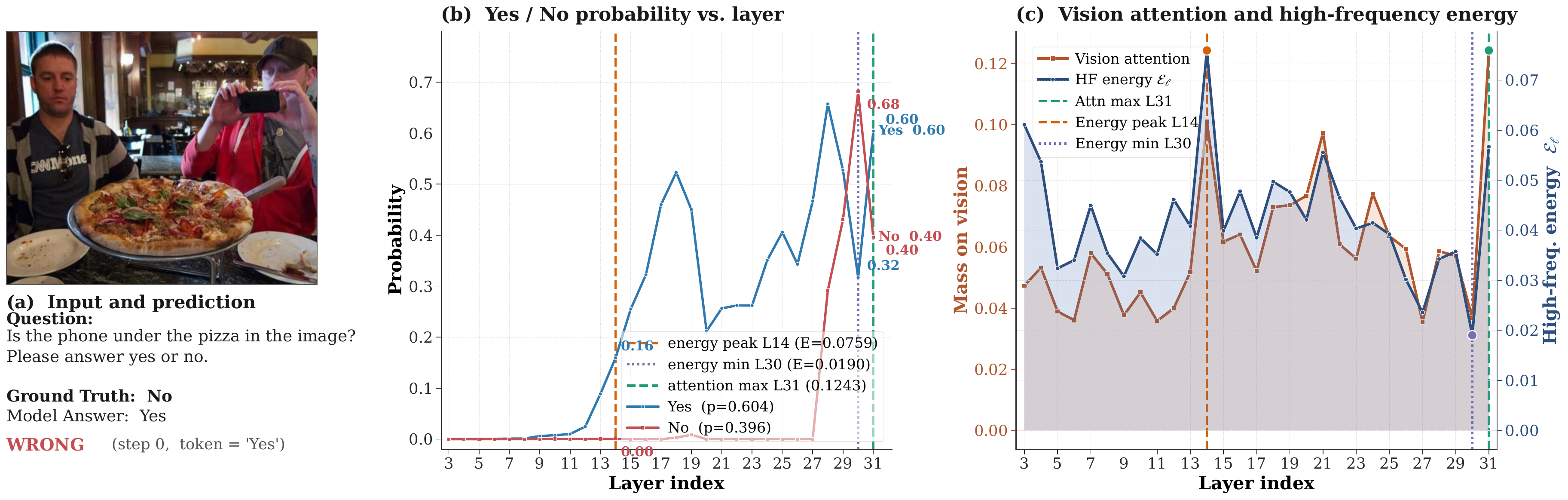}
    \caption{A case study of hallucination during decoding progress.}
    \label{fig:vis5}
\end{figure}

\begin{figure}[ht]
    \centering
    \includegraphics[width=1.0\linewidth]{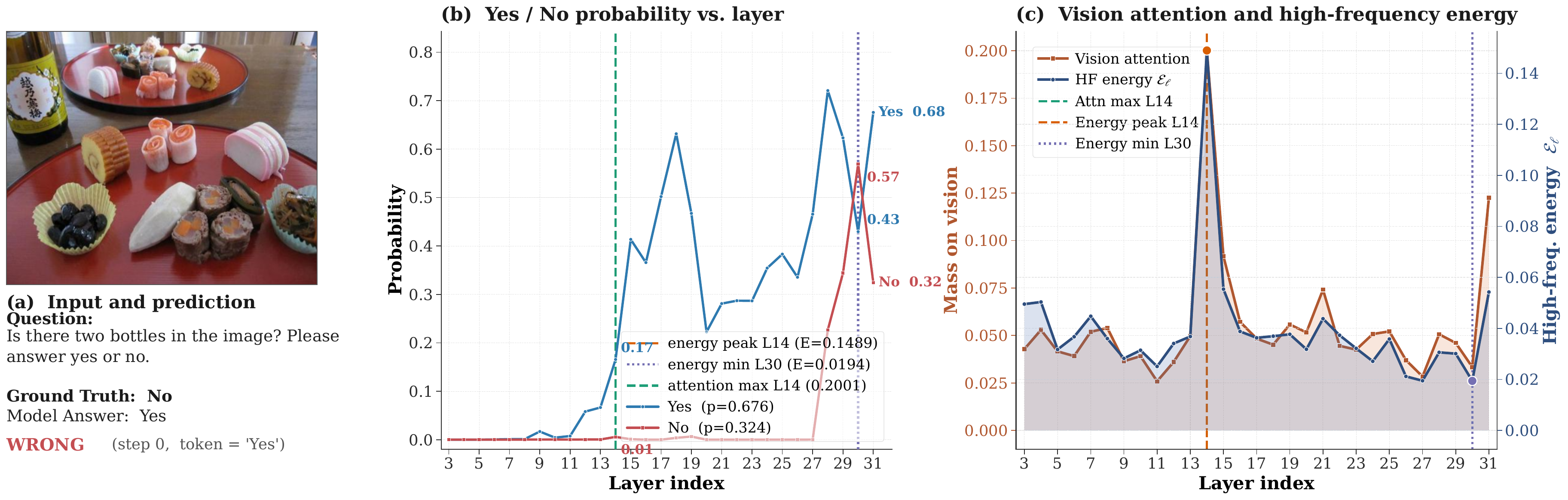}
    \caption{A case study of hallucination during decoding progress.}
    \label{fig:vis6}
\end{figure}

\begin{figure}[ht]
    \centering
    \includegraphics[width=1.0\linewidth]{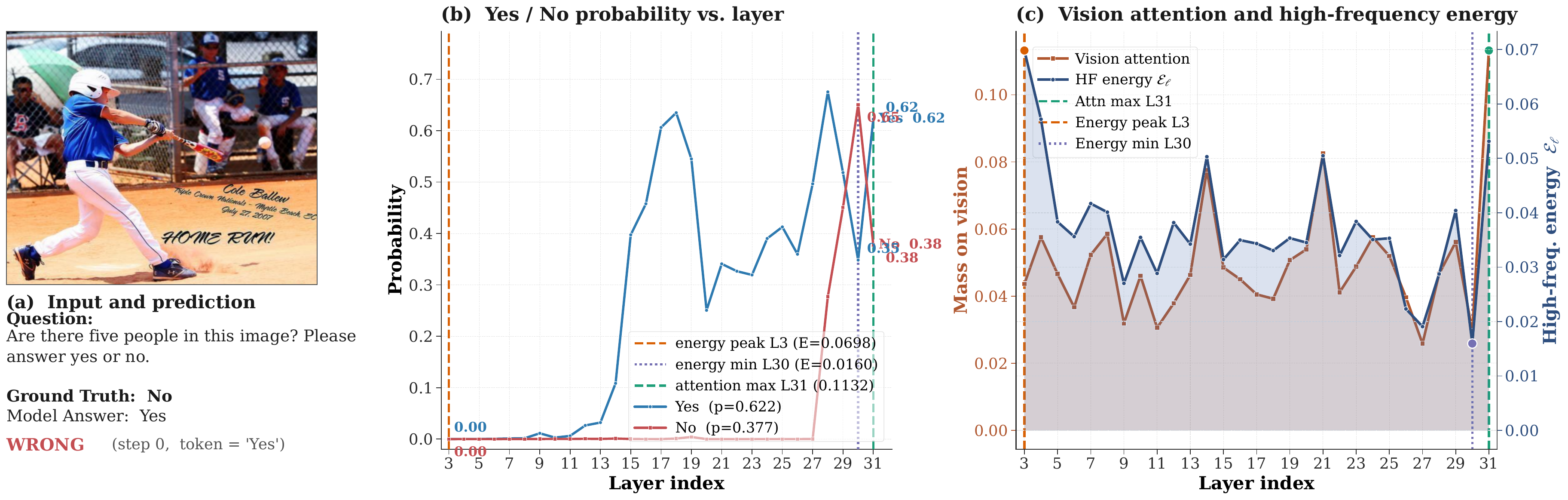}
    \caption{A case study of hallucination during decoding progress.}
    \label{fig:vis7}
\end{figure}

%% file: references.bib
@misc{liu2024llavanext,
  title={Llavanext: Improved reasoning, ocr, and world knowledge},
  author={Liu, Haotian and Li, Chunyuan and Li, Yuheng and Li, Bo and Zhang, Yuanhan and Shen, Sheng and Lee, Yong Jae},
  year={2024}
}

@article{liu2023improved,
  title={Improved baselines with visual instruction tuning},
  author={Liu, Haotian and Li, Chunyuan and Li, Yuheng and Lee, Yong Jae},
  journal={arXiv preprint arXiv:2310.03744},
  year={2023}
}

@article{glm2024chatglm,
  title={Chatglm: A family of large language models from glm-130b to glm-4 all tools},
  author={Glm, Team and Zeng, Aohan and Xu, Bin and Wang, Bowen and Zhang, Chenhui and Yin, Da and Zhang, Dan and Rojas, Diego and Feng, Guanyu and Zhao, Hanlin and others},
  journal={arXiv preprint arXiv:2406.12793},
  year={2024}
}

@article{xing2024mitigating,
  title={Mitigating object hallucination via concentric causal attention},
  author={Xing, Yun and Li, Yiheng and Laptev, Ivan and Lu, Shijian},
  journal={arXiv preprint arXiv:2410.15926},
  year={2024}
}

@article{neo2024vord,
  title={VORD: Visual Ordinal Calibration for Mitigating Object Hallucinations in Large Vision-Language Models},
  author={Neo, Dexter and Chen, Tsuhan},
  journal={arXiv preprint arXiv:2412.15739},
  year={2024}
}

@article{qu2024alleviating,
  title={Alleviating hallucination in large vision-language models with active retrieval augmentation},
  author={Qu, Xiaoye and Chen, Qiyuan and Wei, Wei and Sun, Jishuo and Dong, Jianfeng},
  journal={arXiv preprint arXiv:2408.00555},
  year={2024}
}

@article{zhang2024seeing,
  title={Seeing clearly by layer two: Enhancing attention heads to alleviate hallucination in lvlms},
  author={Zhang, Xiaofeng and Quan, Yihao and Gu, Chaochen and Shen, Chen and Yuan, Xiaosong and Yan, Shaotian and Cheng, Hao and Wu, Kaijie and Ye, Jieping},
  journal={arXiv preprint arXiv:2411.09968},
  year={2024}
}

@article{wang2024mitigating,
  title={Mitigating hallucinations in large vision-language models with instruction contrastive decoding},
  author={Wang, Xintong and Pan, Jingheng and Ding, Liang and Biemann, Chris},
  journal={arXiv preprint arXiv:2403.18715},
  year={2024}
}

@article{huang2024visualhall,
  title={Visual hallucinations of multi-modal large language models},
  author={Huang, Wen and Liu, Hongbin and Guo, Minxin and Gong, Neil Zhenqiang},
  journal={arXiv preprint arXiv:2402.14683},
  year={2024}
}

@inproceedings{yu2024rlhf,
  title={Rlhf-v: Towards trustworthy mllms via behavior alignment from fine-grained correctional human feedback},
  author={Yu, Tianyu and Yao, Yuan and Zhang, Haoye and He, Taiwen and Han, Yifeng and Cui, Ganqu and Hu, Jinyi and Liu, Zhiyuan and Zheng, Hai-Tao and Sun, Maosong and others},
  booktitle={Proceedings of the IEEE/CVF Conference on Computer Vision and Pattern Recognition},
  pages={13807--13816},
  year={2024}
}

@article{lin2024hassurvey,
  title={Has Multimodal Learning Delivered Universal Intelligence in Healthcare? A Comprehensive Survey},
  author={Lin, Qika and Zhu, Yifan and Mei, Xin and Huang, Ling and Ma, Jingying and He, Kai and Peng, Zhen and Cambria, Erik and Feng, Mengling},
  journal={arXiv preprint arXiv:2408.12880},
  year={2024}
}

@inproceedings{ding2024holistic,
  title={Holistic Autonomous Driving Understanding by Bird's-Eye-View Injected Multi-Modal Large Models},
  author={Ding, Xinpeng and Han, Jianhua and Xu, Hang and Liang, Xiaodan and Zhang, Wei and Li, Xiaomeng},
  booktitle={Proceedings of the IEEE/CVF Conference on Computer Vision and Pattern Recognition},
  pages={13668--13677},
  year={2024}
}

@inproceedings{yang2024rag,
  title={IM-RAG: Multi-Round Retrieval-Augmented Generation Through Learning Inner Monologues},
  author={Yang, Diji and Rao, Jinmeng and Chen, Kezhen and Guo, Xiaoyuan and Zhang, Yawen and Yang, Jie and Zhang, Yi},
  booktitle={Proceedings of the 47th International ACM SIGIR Conference on Research and Development in Information Retrieval},
  pages={730--740},
  year={2024}
}

@inproceedings{
    liu2024mitigating,
    title={Mitigating Hallucination in Large Multi-Modal Models via Robust Instruction Tuning},
    author={Fuxiao Liu and Kevin Lin and Linjie Li and Jianfeng Wang and Yaser Yacoob and Lijuan Wang},
    booktitle={The Twelfth International Conference on Learning Representations},
    year={2024},
    url={https://openreview.net/forum?id=J44HfH4JCg}
}

@inproceedings{li2023contrastive,
  title={Contrastive Decoding: Open-ended Text Generation as Optimization},
  author={Li, Xiang Lisa and Holtzman, Ari and Fried, Daniel and Liang, Percy and Eisner, Jason and Hashimoto, Tatsunori and Zettlemoyer, Luke and Lewis, Mike},
  booktitle={The 61st Annual Meeting Of The Association For Computational Linguistics},
  year={2023}
}

@inproceedings{shi2024trusting,
  title={Trusting Your Evidence: Hallucinate Less with Context-aware Decoding},
  author={Shi, Weijia and Han, Xiaochuang and Lewis, Mike and Tsvetkov, Yulia and Zettlemoyer, Luke and Yih, Wen-tau},
  booktitle={Proceedings of the 2024 Conference of the North American Chapter of the Association for Computational Linguistics: Human Language Technologies (Volume 2: Short Papers)},
  pages={783--791},
  year={2024}
}

@article{alayrac2022flamingo,
  title={Flamingo: a visual language model for few-shot learning},
  author={Alayrac, Jean-Baptiste and Donahue, Jeff and Luc, Pauline and Miech, Antoine and Barr, Iain and Hasson, Yana and Lenc, Karel and Mensch, Arthur and Millican, Katherine and Reynolds, Malcolm and others},
  journal={Advances in Neural Information Processing Systems},
  volume={35},
  pages={23716--23736},
  year={2022}
}

@inproceedings{rohrbach2018object,
  title={Object Hallucination in Image Captioning},
  author={Rohrbach, Anna and Hendricks, Lisa Anne and Burns, Kaylee and Darrell, Trevor and Saenko, Kate},
  booktitle={Proceedings of the 2018 Conference on Empirical Methods in Natural Language Processing},
  pages={4035--4045},
  year={2018}
}

@inproceedings{gunjal2024detecting,
  title={Detecting and preventing hallucinations in large vision language models},
  author={Gunjal, Anisha and Yin, Jihan and Bas, Erhan},
  booktitle={Proceedings of the AAAI Conference on Artificial Intelligence},
  volume={38},
  pages={18135--18143},
  year={2024}
}

@inproceedings{zhouanalyzing,
  title={Analyzing and Mitigating Object Hallucination in Large Vision-Language Models},
  author={Zhou, Yiyang and Cui, Chenhang and Yoon, Jaehong and Zhang, Linjun and Deng, Zhun and Finn, Chelsea and Bansal, Mohit and Yao, Huaxiu},
  booktitle={International Conference on Learning Representations},
  year={2024}
}

@inproceedings{chenhalc,
  title={HALC: Object Hallucination Reduction via Adaptive Focal-Contrast Decoding},
  author={Chen, Zhaorun and Zhao, Zhuokai and Luo, Hongyin and Yao, Huaxiu and Li, Bo and Zhou, Jiawei},
  booktitle={Forty-first International Conference on Machine Learning},
  year={2024}
}

@inproceedings{huo2024self,
  title={Self-introspective decoding: Alleviating hallucinations for large vision-language models},
  author={Huo, Fushuo and Xu, Wenchao and Zhang, Zhong and Wang, Haozhao and Chen, Zhicheng and Zhao, Peilin},
  booktitle={ICLR},
  year={2024}
}

@inproceedings{liu2024paying,
  title={Paying more attention to image: A training-free method for alleviating hallucination in lvlms},
  author={Liu, Shi and Zheng, Kecheng and Chen, Wei},
  booktitle={European Conference on Computer Vision},
  pages={125--140},
  year={2024},
  organization={Springer}
}

@inproceedings{li2023blip,
  title={Blip-2: Bootstrapping language-image pre-training with frozen image encoders and large language models},
  author={Li, Junnan and Li, Dongxu and Savarese, Silvio and Hoi, Steven},
  booktitle={International Conference on Machine Learning},
  pages={19730--19742},
  year={2023},
  organization={PMLR}
}

@inproceedings{lin2014microsoft,
  title={Microsoft coco: Common objects in context},
  author={Lin, Tsung-Yi and Maire, Michael and Belongie, Serge and Hays, James and Perona, Pietro and Ramanan, Deva and Doll{\'a}r, Piotr and Zitnick, C Lawrence},
  booktitle={Proceedings of the European Conference on Computer Vision (ECCV)},
  year={2014},
}

@inproceedings{li2023evaluating,
  title={Evaluating Object Hallucination in Large Vision-Language Models},
  author={Yifan Li and Yifan Du and Kun Zhou and Jinpeng Wang and Xin Zhao and Ji-Rong Wen},
  booktitle={The 2023 Conference on Empirical Methods in Natural Language Processing},
  year={2023}
}

@article{fu2023mme,
  title={MME: A Comprehensive Evaluation Benchmark for Multimodal Large Language Models},
  author={Fu, Chaoyou and Chen, Peixian and Shen, Yunhang and Qin, Yulei and Zhang, Mengdan and Lin, Xu and Yang, Jinrui and Zheng, Xiawu and Li, Ke and Sun, Xing and others},
  journal={arXiv preprint arXiv:2306.13394},
  year={2023}
}

@inproceedings{2024MMVet,
  title={MM-Vet: Evaluating Large Multimodal Models for Integrated Capabilities},
  author={Yu, Weihao and Yang, Zhengyuan and Li, Linjie and Wang, Jianfeng and Lin, Kevin and Liu, Zicheng and Wang, Xinchao and Wang, Lijuan},
  booktitle={Forty-first International Conference on Machine Learning},
  year={2024}
}

@inproceedings{guan2024hallusionbench,
  title={Hallusionbench: an advanced diagnostic suite for entangled language hallucination and visual illusion in large vision-language models},
  author={Guan, Tianrui and Liu, Fuxiao and Wu, Xiyang and Xian, Ruiqi and Li, Zongxia and Liu, Xiaoyu and Wang, Xijun and Chen, Lichang and Huang, Furong and Yacoob, Yaser and others},
  booktitle={Proceedings of the IEEE/CVF Conference on Computer Vision and Pattern Recognition},
  pages={14375--14385},
  year={2024}
}

@article{lyu2025realrag,
  title={RealRAG: Retrieval-augmented Realistic Image Generation via Self-reflective Contrastive Learning},
  author={Lyu, Yuanhuiyi and Zheng, Xu and Jiang, Lutao and Yan, Yibo and Zou, Xin and Zhou, Huiyu and Zhang, Linfeng and Hu, Xuming},
  journal={arXiv preprint arXiv:2502.00848},
  year={2025}
}

@article{zheng2024reefknot,
  title={Reefknot: A Comprehensive Benchmark for Relation Hallucination Evaluation, Analysis and Mitigation in Multimodal Large Language Models},
  author={Zheng, Kening and Chen, Junkai and Yan, Yibo and Zou, Xin and Hu, Xuming},
  journal={arXiv preprint arXiv:2408.09429},
  year={2024}
}

@article{liu2024visual,
  title={Visual instruction tuning},
  author={Liu, Haotian and Li, Chunyuan and Wu, Qingyang and Lee, Yong Jae},
  journal={Advances in Neural Information Processing Systems},
  volume={36},
  year={2024}
}

@article{bai2023qwen,
  title={Qwen-vl: A frontier large vision-language model with versatile abilities},
  author={Bai, Jinze and Bai, Shuai and Yang, Shusheng and Wang, Shijie and Tan, Sinan and Wang, Peng and Lin, Junyang and Zhou, Chang and Zhou, Jingren},
  journal={arXiv preprint arXiv:2308.12966},
  year={2023}
}

@inproceedings{chuang2023dola,
  title={DoLa: Decoding by Contrasting Layers Improves Factuality in Large Language Models},
  author={Chuang, Yung-Sung and Xie, Yujia and Luo, Hongyin and Kim, Yoon and Glass, James R and He, Pengcheng},
  booktitle={The Twelfth International Conference on Learning Representations},
  year={2023}
}

@inproceedings{vcd2024cvpr,
  title={Mitigating object hallucinations in large vision-language models through visual contrastive decoding},
  author={Leng, Sicong and Zhang, Hang and Chen, Guanzheng and Li, Xin and Lu, Shijian and Miao, Chunyan and Bing, Lidong},
  booktitle={Proceedings of the IEEE/CVF Conference on Computer Vision and Pattern Recognition},
  pages={13872--13882},
  year={2024}
}

@inproceedings{opera2024cvpr,
  title={Opera: Alleviating hallucination in multi-modal large language models via over-trust penalty and retrospection-allocation},
  author={Huang, Qidong and Dong, Xiaoyi and Zhang, Pan and Wang, Bin and He, Conghui and Wang, Jiaqi and Lin, Dahua and Zhang, Weiming and Yu, Nenghai},
  booktitle={Proceedings of the IEEE/CVF Conference on Computer Vision and Pattern Recognition},
  pages={13418--13427},
  year={2024}
}

@article{bai2024hallucination,
  title={Hallucination of multimodal large language models: A survey},
  author={Bai, Zechen and Wang, Pichao and Xiao, Tianjun and He, Tong and Han, Zongbo and Zhang, Zheng and Shou, Mike Zheng},
  journal={arXiv preprint arXiv:2404.18930},
  year={2024}
}

@article{yin2023survey,
  title={A survey on multimodal large language models},
  author={Yin, Shukang and Fu, Chaoyou and Zhao, Sirui and Li, Ke and Sun, Xing and Xu, Tong and Chen, Enhong},
  journal={arXiv preprint arXiv:2306.13549},
  year={2023}
}

@article{zoulook,
  title={Look Twice Before You Answer: Memory-Space Visual Retracing for Hallucination Mitigation in Multimodal Large Language Models}, 
  author={Zou, Xin and Wang, Yizhou and Yan, Yibo and Lyu, Yuanhuiyi and Zheng, Kening and Huang, Sirui and Chen, Junkai and Jiang, Peijie and Liu, Jia and Tang, Chang and Hu, Xuming},
  journal={Forty-second International Conference on Machine Learning (ICML)},
  year={2025}
}

@inproceedings{
wang2025mllm,
title={{MLLM} can see? Dynamic Correction Decoding for Hallucination Mitigation},
author={Chenxi Wang and Xiang Chen and Ningyu Zhang and Bozhong Tian and Haoming Xu and Shumin Deng and Huajun Chen},
booktitle={The Thirteenth International Conference on Learning Representations},
year={2025},
url={https://openreview.net/forum?id=4z3IguA4Zg}
}

@article{videohallucer,
    title={VideoHallucer: Evaluating Intrinsic and Extrinsic Hallucinations in Large Video-Language Models},
    author={Wang, Yuxuan and Wang, Yueqian and Zhao, Dongyan and Xie, Cihang and Zheng, Zilong},
    journal={arxiv},
    year={2024}
}

@inproceedings{lin-etal-2024-video,
    title = "Video-{LL}a{VA}: Learning United Visual Representation by Alignment Before Projection",
    author = "Lin, Bin  and
      Ye, Yang  and
      Zhu, Bin  and
      Cui, Jiaxi  and
      Ning, Munan  and
      Jin, Peng  and
      Yuan, Li",
    editor = "Al-Onaizan, Yaser  and
      Bansal, Mohit  and
      Chen, Yun-Nung",
    booktitle = "Proceedings of the 2024 Conference on Empirical Methods in Natural Language Processing",
    month = nov,
    year = "2024",
    address = "Miami, Florida, USA",
    publisher = "Association for Computational Linguistics",
    url = "https://aclanthology.org/2024.emnlp-main.342/",
    doi = "10.18653/v1/2024.emnlp-main.342",
    pages = "5971--5984",
}

@article{qi2026detecting,
  title={Detecting Contextual Hallucinations in LLMs with Frequency-Aware Attention},
  author={Qi, Siya and Chen, Yudong and Zhao, Runcong and Zhu, Qinglin and Hu, Zhanghao and Liu, Wei and He, Yulan and Yuan, Zheng and Gui, Lin},
  journal={arXiv preprint arXiv:2602.18145},
  year={2026}
}

@inproceedings{
zhu2026look,
title={Look Carefully: Adaptive Visual Reinforcements in Multimodal Large Language Models for Hallucination Mitigation},
author={Xingyu Zhu and Kesen Zhao and Liang Yi and Shuo Wang and Zhicai Wang and Beier Zhu and Hanwang Zhang and Xiangnan He},
booktitle={The Fourteenth International Conference on Learning Representations},
year={2026},
url={https://openreview.net/forum?id=KhEjFE1QvQ}
}

@article{ren2026nolan,
  title={NoLan: Mitigating Object Hallucinations in Large Vision-Language Models via Dynamic Suppression of Language Priors},
  author={Ren, Lingfeng and Yu, Weihao and Yu, Runpeng and Wang, Xinchao},
  journal={arXiv preprint arXiv:2602.22144},
  year={2026}
}

@inproceedings{ASCD,
  title={ASCD: Attention-Steerable Contrastive Decoding for Reducing Hallucination in MLLM},
  author={Wang, Yujun and Bi, Jinhe and Pirk, Soren and Ma, Yunpu and others},
  booktitle={Proceedings of the AAAI Conference on Artificial Intelligence},
  volume={40},
  pages={10306--10314},
  year={2026}
}

@inproceedings{
yin2026dynamic,
title={Dynamic Multimodal Activation Steering for Hallucination Mitigation in Large Vision-Language Models},
author={Jianghao Yin and Qin Chen and Kedi Chen and Jie Zhou and Xingjiao Wu and Liang He},
booktitle={The Fourteenth International Conference on Learning Representations},
year={2026},
url={https://openreview.net/forum?id=YtWZdwEG5K}
}

@inproceedings{iTaD,
    title = "Mitigating Hallucinations in Multi-modal Large Language Models via Image Token Attention-Guided Decoding",
    author = "Xu, Xinhao  and
      Chen, Hui  and
      Lyu, Mengyao  and
      Zhao, Sicheng  and
      Xiong, Yizhe  and
      Lin, Zijia  and
      Han, Jungong  and
      Ding, Guiguang",
    editor = "Chiruzzo, Luis  and
      Ritter, Alan  and
      Wang, Lu",
    booktitle = "Proceedings of the 2025 Conference of the Nations of the Americas Chapter of the Association for Computational Linguistics: Human Language Technologies (Volume 1: Long Papers)",
    year = "2025",
    address = "Albuquerque, New Mexico",
    publisher = "Association for Computational Linguistics",
    url = "https://aclanthology.org/2025.naacl-long.75/",
    doi = "10.18653/v1/2025.naacl-long.75",
    pages = "1571--1590",
    ISBN = "979-8-89176-189-6",
}

@article{song2026seeing,
  title={Seeing Right but Saying Wrong: Inter-and Intra-Layer Refinement in MLLMs without Training},
  author={Song, Shezheng and Li, Shasha and Yu, Jie},
  journal={arXiv preprint arXiv:2601.07359},
  year={2026}
}

@article{pytorch,
  title={Pytorch: An imperative style, high-performance deep learning library},
  author={Paszke, Adam and Gross, Sam and Massa, Francisco and Lerer, Adam and Bradbury, James and Chanan, Gregory and Killeen, Trevor and Lin, Zeming and Gimelshein, Natalia and Antiga, Luca and others},
  journal={Advances in neural information processing systems},
  volume={32},
  year={2019}
}
